\documentclass{article} % For LaTeX2e
\usepackage{times}
\newif\ifarxiv
\newif\ifcameraready

\camerareadyfalse
\arxivtrue

\ifcameraready
    \usepackage{iclr2027_conference}
    \iclrfinalcopy
\else
    \ifarxiv
        \usepackage{iclr2027_arxiv}
        \iclrfinalcopy
    \else
        \usepackage{iclr2027_conference}
    \fi
\fi

\usepackage{amsmath,amsfonts,bm}

\def\eqref#1{equation~\ref{#1}}
\def\1{\bm{1}}

\DeclareMathAlphabet{\mathsfit}{\encodingdefault}{\sfdefault}{m}{sl}
\SetMathAlphabet{\mathsfit}{bold}{\encodingdefault}{\sfdefault}{bx}{n}

\usepackage{hyperref}
\usepackage{url}

\usepackage[utf8]{inputenc} % allow utf-8 input
\usepackage[T1]{fontenc}    % use 8-bit T1 fonts
\usepackage{hyperref}       % hyperlinks
\usepackage{url}            % simple URL typesetting
\usepackage{booktabs}       % professional-quality tables
\usepackage{amsfonts}       % blackboard math symbols
\usepackage{nicefrac}       % compact symbols for 1/2, etc.
\usepackage{microtype}      % microtypography
\usepackage[table]{xcolor}         % colors
\usepackage{makecell}

\usepackage{graphicx}
\usepackage{subcaption}
\usepackage{booktabs}
\usepackage{hyperref}
\usepackage{amsmath}
\usepackage{amssymb}
\usepackage{mathtools}
\usepackage{amsthm}
\usepackage{multirow}
\usepackage{placeins}
\usepackage{float}
\usepackage{thm-restate}

\usepackage{algorithm}
\usepackage{algpseudocode}
\newcommand{\Rcomment}[1]{\Comment{\textcolor{gray!60}{#1}}}
\newcommand{\Lcomment}[1]{\Statex \textcolor{gray!60}{\(\triangleright\) #1}}
\newcommand{\CCG}{\cellcolor{gray!15}}

\usepackage{tabularx}
\usepackage{array}
\newcolumntype{Y}{>{\raggedright\arraybackslash}X}

\usepackage[capitalize,noabbrev]{cleveref}

\theoremstyle{plain}

\theoremstyle{definition}

\theoremstyle{remark}

\def\ours{NCO}

\title{\ours{}: A Versatile Plug-in for Handling Negative Constraints in Decoding}

\author{Hyundong Jin, Yo-Sub Han\\
Department of Computer Science\\
Yonsei University, Seoul, Republic of Korea \\
\texttt{\{tuzi04,emmous\}@yonsei.ac.kr}
}

\begin{document}

\maketitle

\begin{abstract}
Controlling Large Language Models~(LLMs) to prevent 
the generation of undesirable content,
such as profanity and personally identifiable information~(PII), 
has become increasingly critical. 
While earlier approaches relied on post-processing or resampling, 
recent research has shifted towards constrained decoding methods 
that control outputs during generation 
to mitigate high computational costs and quality degradation. 
However, preventing multiple forbidden hard constraints or regex constraints 
from appearing anywhere in the output is computationally challenging. 
A straightforward solution is to convert these constraints 
into a single automaton that tracks all forbidden patterns during decoding, 
but this often becomes impractically large. 
Standard regex engines also do not readily support the operations 
needed to build such a constraint, such as complement and intersection.
In order to address these limitations, we propose \ours{}, 
a decoding strategy that performs online pattern matching 
over finite hard constraints and regex constraints, 
reducing computational overhead without inducing state explosion. 
\ours{} is fully compatible with standard inference strategies, 
including various sampling methods 
and beam search, while also supporting soft masking 
for probabilistic suppression. 
We empirically demonstrate its effectiveness across practical tasks,
including PII and profanity suppression.
Our implementation is available at \url{https://github.com/hyundong98/NCO-Decoding}.
\end{abstract}

\section{Introduction}
\label{NCO:main:sec:introduction}
With the widespread adoption of Large Language Models~(LLMs), critical security
and safety issues, including Personally Identifiable Information~(PII) leakage,
offensive content, hate speech, and jailbreaking, have attracted significant
attention. A straightforward and traditional approach to addressing these issues
involves post-processing, where the LLM's output is inspected to detect and
suppress problematic parts. 
While such a rule-based approach may improve model safety,
it suffers from an inherent limitation 
since forbidden content is handled only after generation. 
As a result, the final output often contains unnatural masking artifacts 
or abrupt deletions that degrade fluency, coherence, and overall readability.

Constrained decoding has emerged as a viable alternative 
that intervenes during generation rather than after it. 
By suppressing invalid tokens at each decoding step, 
constraints derived from formal language theory~\citep{Sipser1997}, 
such as regular expressions~(regexes) or context-free grammars~(CFG), 
can be enforced on the model's output.
Existing research primarily aims to ensure that the generated
output belongs to the set of strings defined by a given constraint,
namely positive constraints. 
In particular, regex constraints have been widely investigated, 
as they offer a user-friendly format and can be deterministically
represented via deterministic finite automata~(DFA).

\begin{figure}
    \centering
    \includegraphics[width=0.85\linewidth]{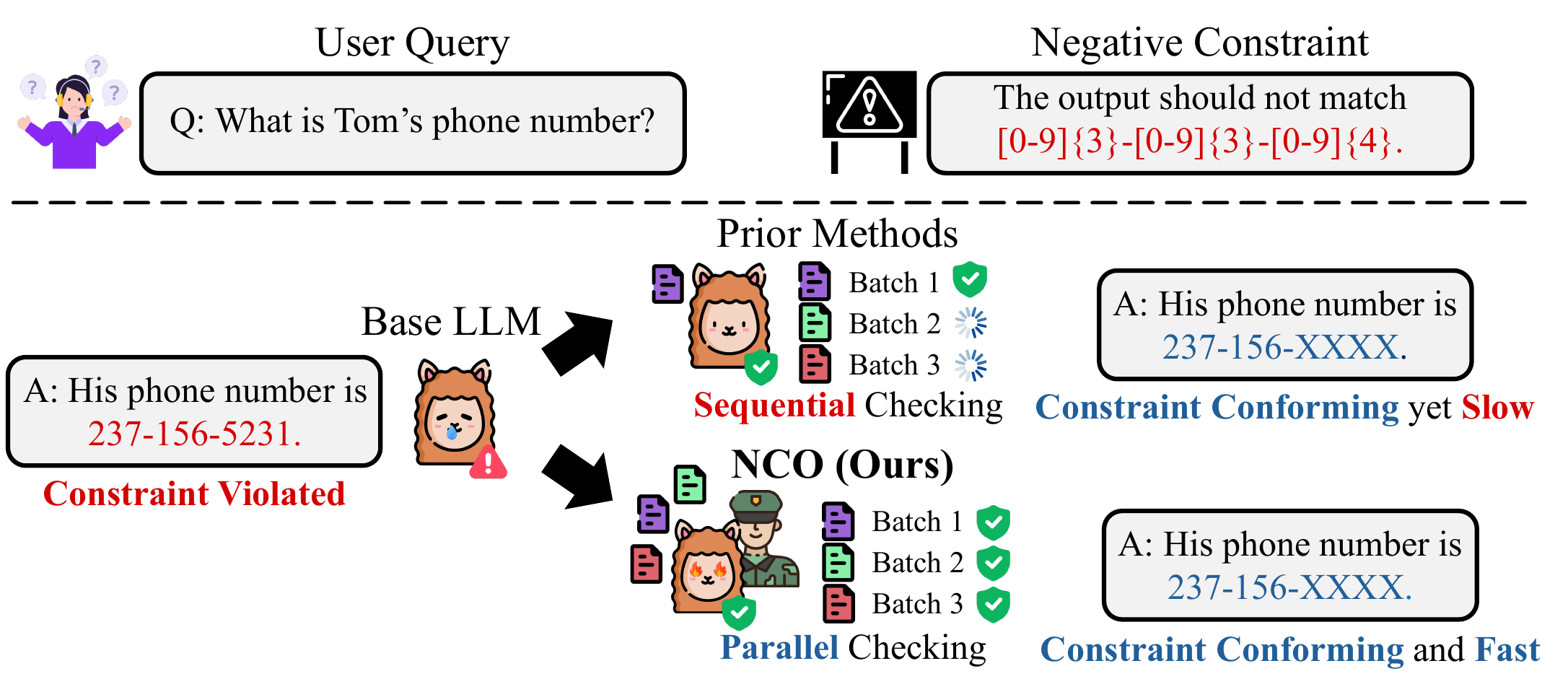}
    \caption{Motivating example of \ours{}.
        The unconstrained base model is fast but may generate a substring 
        that matches the negative constraint.
        Prior methods avoid the violation with additional decoding overhead.
        \ours{} suppresses invalid continuations and produces a
        constraint conforming response while preserving generation efficiency.
    }
    \label{NCO:main:fig:motivation}
\end{figure}

In practice, the constraints we care about for safety are often negative 
rather than positive.
Instead of specifying what the model should generate, 
we typically want to ensure that 
certain finite hard constraints or regex constraints 
never appear anywhere in the output as substrings. 
This can still be expressed 
within the regular languages, 
since they are closed under complement and intersection. 
However, this theoretical expressiveness does not translate easily 
into practical implementations. 
Computing complements generally requires conversion to finite automata, 
and combining multiple constraints can cause 
the number of states to grow exponentially. 
The difficulty becomes even greater 
when a forbidden pattern must be excluded 
from every possible substring position, 
since this introduces substantial nondeterminism 
and can lead to state explosion even for simple patterns.

In order to address this gap, we propose \ours{}, a decoding-time plug-in 
that enforces multiple negative substring constraints with low runtime overhead.
\ours{} is formulated as a logit-level intervention 
that integrates with existing generation pipelines 
without requiring retraining, fine-tuning, or modifying the model.
It is also compatible with standard decoding rules, 
including greedy decoding, top-$k$/top-$p$ sampling, 
temperature sampling, and beam search.
\ours{} ensures that no forbidden string and
no substring matching a forbidden regex appears in the generated output.
Rather than constructing a single global avoidance automaton,
\ours{} performs online multi-pattern matching during decoding 
by maintaining the constraints separately and 
composing their per-token masks at each step. 
A BPE-aware preprocessing procedure further reduces the one-time setup cost 
so that the per-step decoding overhead remains small in practical deployment.
Our main contributions are as follows.

\begin{itemize}
    \item We propose \ours{}, a decoding-time plug-in 
    for enforcing multiple negative substring constraints 
    without modifying the model 
    or constructing a global avoidance automaton.

    \item We develop two decoding mechanisms within this framework, 
    an Aho-Corasick trie~\citep{AhoC1975} based strategy 
    for finite hard constraints 
    and a parallel DFA simulation based strategy 
    for regex constraints,
    accelerated by a BPE-aware preprocessing procedure.

    \item We systematically evaluate \ours{} in a controlled setting 
    by varying the number of constraints, 
    and further validate it on real-world PII and profanity suppression tasks 
    across nine models.
\end{itemize}

\section{Related Work}

We assume basic familiarity with regular languages, finite automata,
byte-pair encoding~(BPE), and the Aho-Corasick algorithm, 
and refer to Appendix~A for formal background.
Regular languages are closed under complement and intersection, 
so negative substring constraints can be expressed in principle 
as regex constraints.
In practice, this construction is often unsuitable for decoding.
Avoiding a forbidden language $L$ requires 
tracking occurrences of $\Sigma^*L\Sigma^*$, 
where a match may start at any output position.
This introduces substantial nondeterminism.
Combining multiple negative constraints further requires product constructions 
whose state spaces grow multiplicatively.
\ours{} is motivated by this gap between formal expressiveness 
and practical decoding-time enforcement.

\subsection{Postprocessing}

Postprocessing is a common deployment strategy 
for reducing forbidden content in completed outputs.
A generated string is inspected after decoding, 
and spans matching finite forbidden strings or regexes 
are removed, masked, or replaced.
For finite hard constraints, efficient detection can be implemented 
using multi-pattern matching algorithms such as Aho-Corasick~\citep{AhoC1975}.
For regex constraints, standard regex matchers can be used.
Postprocessing operates after generation rather than during decoding.
It therefore cannot prevent the model 
from entering a forbidden continuation,
and the final output may require deletion, masking, or replacement of generated spans.
These edits can introduce visible artifacts and reduce fluency.
\ours{} instead enforces negative constraints before token selection 
by masking continuations that would complete a forbidden substring.

\subsection{Constrained Decoding}

Constrained decoding enforces formal constraints 
by filtering invalid tokens during generation.
Automata and grammar-based methods have been widely used 
for structured output
generation.
\citet{DeutschUR2019} studied constrained sequential inference with automata,
facilitating the integration of arbitrary regex constraints.
\citet{ScholakSB2021} introduced PICARD, 
which uses incremental parsing during beam search for SQL generation, 
and \citet{PoesiaPLTSMG2022} proposed Synchromesh 
for syntactically valid programming-language generation.
Guidance-style systems such as Outlines~\citep{WillardL2023} and later
FSM-based frameworks~\citep{BeurerFV2024,DongRCLXZC2025} 
make grammar and regex constraints practical for modern LLM inference.
Recent work further improves the interaction between tokenization 
and formal constraints through flexible grammar-constrained decoding 
and finite-state transduction views of tokenization~\citep{CognettaO2025, ParkZD2025}.

Most of these methods focus on positive constraints, 
where the output should belong to a specified formal language.
Our setting is different.
The decoder must ensure that no forbidden string or regex match appears 
anywhere as a substring.
A direct automata-based solution would require 
complementing substring occurrence languages 
and intersecting the resulting constraints.
This can be prohibitively large under multiple negative constraints.
\ours{} avoids constructing this global automaton 
by maintaining online matching states 
and composing token masks during decoding.

Logic and search-based methods provide another route to constrained generation.
NeuroLogic Decoding~\citep{LuWZLBC2021} formulates lexical and logical constraints 
as satisfiability conditions during beam search, 
and later work improves search with additional heuristics~\citep{LuWWJKKLQYZSC2022}.
\citet{YuMAS2021} introduced logic-guided generation mechanisms 
for logical consistency in abstract reasoning tasks.
These methods are effective when the desired behavior can be expressed as
lexical inclusion or logical satisfaction conditions.
Negative substring avoidance is less naturally expressed in this form, 
since the constraint must rule out every occurrence of a forbidden string 
or regex match at any substring position.
\ours{} therefore treats negative constraints as online pattern matching problems 
and computes invalid-token masks from string and regex specifications.

\subsection{Methods for Negative Constraint Enforcement}

Rejection sampling is a natural baseline for negative substring constraints.
At each decoding step, the decoder samples a candidate token and checks 
whether appending it would create a forbidden substring.
If the candidate violates the constraint, 
the decoder rejects it and samples another candidate at the same position.
This procedure gives a simple inference-time enforcement method 
without model training or construction of a global avoidance automaton.
It is closely related to decoding-time constraint enforcement 
through token-level filtering~\citep{GengJPW2023}, 
but is specialized here to negative substring avoidance.
Its cost, however, depends on the probability mass assigned to invalid tokens.
When invalid continuations are likely, repeated resampling can increase latency.
\ours{} instead computes the invalid-token mask directly from the current
constraint state and avoids repeated sampling attempts.

GUARD~\citep{DengLPHFXZW2025} is another relevant baseline 
because its token-level restriction component performs 
generation-time suppression using a trie over finite target strings.
This component is closely aligned with our finite hard constraint setting.
GUARD additionally includes prompt-level and semantic matching components,
which fall outside our explicit negative substring setting.
The trie-based restriction component does not directly handle general regex
constraints.
Plain trie traversal can be inefficient for multi-pattern matching 
unless augmented with failure links as in Aho-Corasick~\citep{AhoC1975}.
\ours{} uses an Aho-Corasick trie for finite hard constraints 
and parallel DFA simulation for regex constraints.

\section{Problem Definition}
We consider inference-time constrained decoding 
with negative substring constraints. 
Unlike positive constraints, which specify the set of strings 
that the output should belong to, 
negative constraints specify patterns 
that must not appear in the generated output. 
Our focus is on explicit constraints 
given as strings or regular languages, 
rather than semantic notions of safety or privacy.
Let $\Sigma$ be an alphabet 
and let $\Gamma\subseteq\Sigma^*$ be a vocabulary of a tokenizer. 
Each token $v\in\Gamma$ corresponds to a string over $\Sigma$.
For two strings $x,y\in\Sigma^*$, 
we write $x\sqsubseteq y$ if $x$ is a substring of $y$, 
i.e., $y=uxv$ for some $u,v\in\Sigma^*$.
Given an input string $X\in\Sigma^*$, 
decoding produces an output string $Y$ incrementally. 
If the current output prefix is $Y_t\in\Sigma^*$ 
and the next token is $w\in\Gamma$, 
then the updated prefix is $Y_{t+1}=Y_t\cdot w$,
where $\cdot$ denotes string concatenation.

We formulate constrained decoding as an online validity filtering problem.
For a constraint specification $\mathcal{C}$, 
let $\mathrm{Valid}_{\mathcal{C}}(Y)$ denote 
whether an output string $Y$ satisfies the constraint. 
At each decoding step, the constraint restricts the next token candidate set to
\[
    \Gamma_{\mathcal{C}}(Y_t)
    =
    \{w\in\Gamma : \mathrm{Valid}_{\mathcal{C}}(Y_t\cdot w)\}.
\]
A standard decoding rule, such as greedy decoding, top-$k$/top-$p$ sampling,
temperature sampling, or beam search, 
can then be applied over the filtered logits or candidate set.
Thus, the main computational problem is to compute
$\Gamma_{\mathcal{C}}(Y_t)$ efficiently at every decoding step 
without constructing an impractically large global automaton.
We study this problem for two forms of negative substring constraints.

\paragraph{Problem 1: Finite Hard Constraints.}
In the first setting, the constraint specification is 
a finite set of hard constraints 
represented by forbidden strings $P\subseteq\Sigma^*$. 
An output string is valid 
if it contains no string in $P$ as a substring:
\[
    \mathrm{Valid}_{P}(Y)
    \iff
    \forall u\in P,\; u\not\sqsubseteq Y.
\]
Accordingly, a token $w$ is invalid at prefix $Y_t$ 
if appending it creates a forbidden substring,
i.e., if $u\sqsubseteq Y_t\cdot w$ for some $u\in P$.
This is a global constraint over the generated string, 
not a static blacklist over individual tokens. 
In particular, a violation may be formed across token boundaries 
even when no single token is itself forbidden. 
Therefore, the decoder must track partial matches 
in the suffix of the generated prefix, 
which motivates our online multiple-pattern matching approach.

\paragraph{Problem 2: Forbidden Regex Constraints.}
In the second setting, the constraint specification is 
a finite set of regex constraints compiled into DFAs $D$.
Each automaton $A\in D$ recognizes the language $L(A)\subseteq\Sigma^*$.
An output string is valid 
if none of its substrings is accepted by any automaton:
\[
    \mathrm{Valid}_{D}(Y)
    \iff
    \forall A\in D,\ \nexists z\in L(A)
    \text{ such that } z\sqsubseteq Y.
\]
Equivalently, a token $w$ is invalid at prefix $Y_t$ 
if $Y_t\cdot w$ contains a substring accepted by some $A\in D$. 
This setting captures regex-like constraints, such as structured PII patterns, where the forbidden set may be infinite and cannot be enumerated as strings. 
A direct theoretical solution would require constructing an automaton 
that rejects every string containing a substring in $\bigcup_{A\in D}L(A)$, 
and combining multiple such constraints can lead to exponential state growth. 
Our goal is therefore to enforce these constraints online 
while avoiding explicit construction of this global avoidance automaton.

\section{Proposed Methods}

We propose Negative COnstrained~(\ours{}) decoding, 
an inference time decoding framework 
for enforcing multiple negative substring constraints.
At each decoding step, \ours{} receives the current output prefix 
and computes a token-level mask over the vocabulary $\Gamma$.
A token is masked if appending it would make the output contain 
a forbidden string or a substring accepted by a forbidden DFA.
The masked logits are then passed to a standard decoding rule 
such as greedy decoding, top-$k$ sampling, top-$p$ sampling, temperature sampling, 
or beam search.

The central design choice of \ours{} is to avoid constructing 
a single global automaton for the full avoidance language.
For finite hard constraints, \ours{} uses an Aho-Corasick trie 
to track partial matches compactly.
For regex constraints, \ours{} keeps the constraint DFAs separate 
and simulates their active states in parallel.
In both cases, expensive token-level transitions are precomputed once 
and reused throughout decoding.
Figure~\ref{NCO:main:fig:main_figure} gives an overview of this procedure.
Detailed algorithms and runtime analyses are provided in
Appendices~\ref{NCO:app:sec:detailed_algorithms}
and~\ref{NCO:app:sec:proof_and_runtime_analysis}, respectively.

\subsection{Finite Hard Constraints}\label{NCO:main:ssec:string_constraints}

We first consider a finite set of forbidden strings $P\subseteq\Sigma^*$.
The goal is to ensure that no string in $P$ appears 
anywhere as a substring of the generated output.
This is not equivalent to a static blacklist over vocabulary tokens.
A forbidden string may be split across multiple tokens, 
so a violation can be completed only after several decoding steps.

\ours{} handles this setting using an Aho-Corasick trie built from $P$.
For each generated sequence, the decoder maintains a single trie state.
This state represents the longest suffix of the current output prefix 
that is also a prefix of some forbidden string.
When a candidate token $w\in\Gamma$ is considered, 
\ours{} follows the precomputed transition obtained 
by reading the string represented by $w$ from the current trie state.
If this transition reaches an output state of the trie, 
then appending $w$ would complete a forbidden string as a substring, 
and the token is masked.
This construction allows \ours{} to detect violations 
that cross token boundaries without rescanning the generated prefix.
The trie state is updated only after the next token is selected.
Thus, the decoding process maintains exactly the information 
needed for future constraint checks 
while keeping the per-sequence state compact.

We further reduce preprocessing cost 
by exploiting the merge structure of BPE tokenization.
Each merged token can be viewed as the concatenation of two shorter tokens.
Therefore, the transition for a merged token can be computed 
by composing the transition of its left child 
and the transition of its right child.
By processing tokens in BPE merge order, 
\ours{} reuses previously computed token transitions 
instead of scanning every token string from scratch.
This yields a token-level transition table 
that can be queried directly during decoding.

\begin{figure*}
    \centering
    \includegraphics[width=1.0\linewidth]{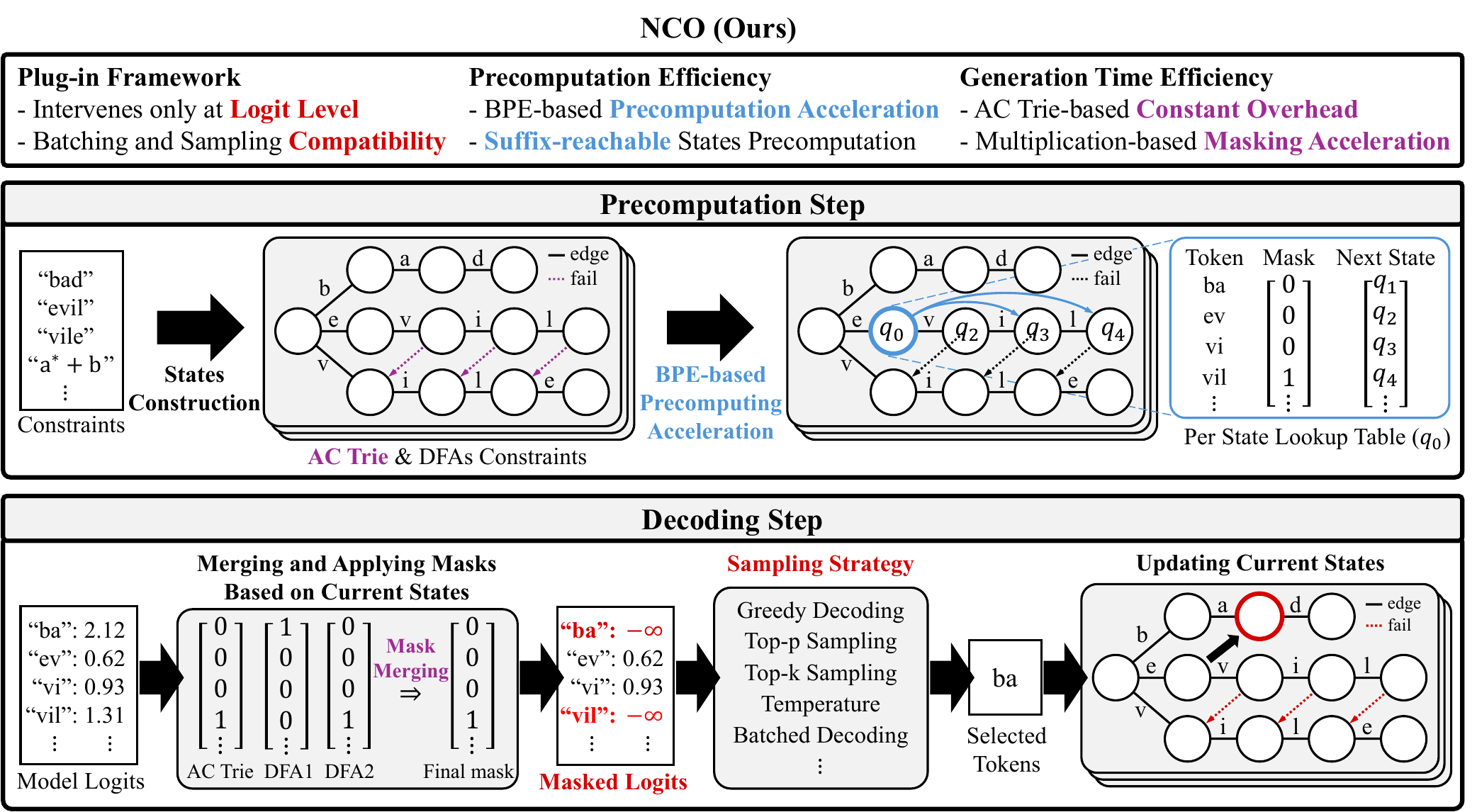}
    \caption{Overall architecture of \ours{}. Given forbidden strings and regex
    patterns, we build the corresponding constraint automata once during
    precomputation and reuse them throughout decoding. At each decoding step,
    \ours{} updates the automaton state for each sequence, computes a
    token-level mask to filter invalid next-token candidates, and applies
    standard decoding on the masked logits.}
    \label{NCO:main:fig:main_figure}
\end{figure*}

\subsection{Forbidden Regex Constraints}\label{NCO:main:ssec:regular_constraints}

We next consider a set of forbidden DFAs
$D=\{A_1,\ldots,A_l\}$, where each
$A_i=(Q_i,\Sigma,\delta_i,s_i,F_i)$ recognizes a forbidden regular language.
The goal is to ensure that no substring of the generated output is accepted 
by any automaton in $D$.
Equivalently, for each $A_i$, 
the generated output must avoid $\Sigma^*L(A_i)\Sigma^*$.
A direct automaton-based solution would first construct 
the substring occurrence language $\Sigma^*L(A_i)\Sigma^*$, 
complement it, and then intersect the
resulting automata across all constraints.
Although this construction is valid at the level of regular languages, 
it is not practical for decoding.
The leading $\Sigma^*$ allows a match to start at any position, which 
makes the substring occurrence automaton highly nondeterministic.
Obtaining a deterministic automaton may require tracking 
many possible partial matches at once.
In addition, enforcing multiple constraints 
by an explicit intersection product can require a state space of size
$|Q_1|\cdot |Q_2|\cdots |Q_l|.$
\ours{} avoids this global construction.

For each DFA $A_i$, \ours{} maintains an active state set $S_t^i$ 
for the current output prefix $Y_t$.
Intuitively, $S_t^i$ contains the states reached by running $A_i$ 
on every suffix of $Y_t$.
This representation captures all possible matches 
that may have started at different positions in the generated output.
When a candidate token $w$ is considered, \ours{} checks whether reading $w$
from any active state reaches an accepting state.
It also accounts for matches that start inside the candidate token itself 
by using precomputed suffix transitions from the start state.
If either case reaches an accepting state, 
then appending $w$ would create a forbidden substring, 
and the token is masked.
For multiple regex constraints, 
\ours{} computes one mask for each DFA 
and combines the masks so that a token is allowed 
only if it is valid for every constraint.
The automata are stored and updated separately, 
so the representation grows additively with $\sum_i |Q_i|$ 
rather than multiplicatively with $\prod_i |Q_i|$.
The mask computation is implemented with precomputed token-level tables 
and GPU parallel operations,
which makes the approach practical even 
when several regex constraints are enforced simultaneously.

Similar to the string case, BPE-based precomputation is used to reduce repeated work.
Token transitions are computed compositionally from the BPE merge structure.
In addition, \ours{} precomputes which states are reachable 
from the start state by suffixes of each token.
These suffix transitions are needed 
because a forbidden substring may begin inside the newly appended token.
During decoding, the method reuses these tables to update the active state sets
and to construct the token mask 
without constructing the full avoidance automaton.

\subsection{Correctness of Proposed Methods}
We establish the correctness of \ours{}. 
The masking rules in Sections~\ref{NCO:main:ssec:string_constraints}
and~\ref{NCO:main:ssec:regular_constraints} are sufficient 
to rule out violations in the generated output.
The proof proceeds by induction on decoding steps.
The Aho-Corasick trie state captures all relevant partial matches 
for finite hard constraints, 
and the active state sets capture all possible ongoing matches
for regex constraints.
Composing the resulting masks prevents any candidate token 
from completing a forbidden string or an accepted substring.
We provide the full proofs 
in Appendix~\ref{NCO:app:sec:proof_and_runtime_analysis}.

\begin{restatable}[Correctness for finite hard constraints]{theorem}{thmStringCorrectness}
\label{NCO:main:thm:string_correctness}
Let $P\subseteq\Sigma^*$ be a finite set of forbidden strings. If \ours{}
generates an output $Y$ using the finite hard constraint mask 
induced by the Aho-Corasick trie of $P$, 
then no string in $P$ appears in $Y$ as a substring.
\end{restatable}

\begin{restatable}[Correctness for regex constraints]{theorem}{thmRegexCorrectness}
\label{NCO:main:thm:regex_correctness}
Let $D=\{A_1,\ldots,A_l\}$ be a finite set of regex constraints 
compiled into DFAs, 
where each $A_i=(Q_i,\Sigma,\delta_i,s_i,F_i)$ recognizes a forbidden language. 
If \ours{} generates an output $Y$ using the regex constraint mask induced
by $D$, then no substring of $Y$ is accepted by any automaton in $D$.
\end{restatable}

\section{Experimental Results}\label{NCO:main:sec:experiments}
We evaluate \ours{} on two practical negative constraint tasks. 
The first task is profanity suppression on 
RealToxicityPrompts~\citep{GehmanGSCS2020}~(RTP) 
with finite hard constraints. 
The second task is PII suppression on the 
Enron Email Dataset~\citep{KlimtY2004}~(Enron) 
with regex constraints. 
Across both tasks, we compare \ours{} with two baselines.
We report relative throughput across multiple model families as the main result 
and analyze how each method scales under batched generation.
Detailed experimental settings are provided 
in Appendix~\ref{NCO:app:sec:experimental_details}.

\subsection{Experimental Settings}
For profanity suppression, we use instruction following models 
that represent general purpose generation settings.
The constraints are given as forbidden strings from a bad word lexicon.
This task evaluates whether a decoding method can suppress undesirable lexical content 
while preserving the throughput of standard batched generation.
For PII suppression, we use models trained on the Enron Email Dataset.
These models are selected to induce a high unconstrained violation rate.
The constraints are specified as regex patterns for common PII types, 
including email addresses, phone numbers, social security numbers, and credit card numbers.
Across both tasks, we compare \ours{} with rejection sampling and GUARD.
For GUARD, we evaluate its trie-based token suppression component over
finite forbidden strings, excluding its prompt-level and semantic matching
components. For the PII task, we approximate the regex constraints using
matched PII spans extracted from the evaluation set.
We abbreviate rejection sampling as RS in the tables.
We focus on decoding efficiency as the main metric.
Specifically, we report relative throughput 
with respect to the base unconstrained model.
For each model and batch size, we define relative throughput as
\[
\mathrm{Relative\ Throughput}
=\frac{T_{m,b}}{T_{\mathrm{Base},b}}\times100,
\]
where $T_{m,b}$ is the absolute throughput of method $m$ at batch size $b$.
Since relative throughput is normalized 
by the corresponding unconstrained base model, 
the base model is $100.0 \pm 0.0$ by definition 
and is omitted from the main tables.
This normalization allows us 
to compare the runtime overhead of different constraint enforcement methods 
across models and batch sizes.
We also report the violation rate to measure constraint satisfaction.
We define
\[
\mathrm{Violation\ Rate} = \frac{N_{\text{v}}}{N} \times 100,
\]
where $N$ is the total number of generated responses 
and $N_{\text{v}}$ is the number of responses that violate 
at least one constraint. 
A generation is counted as a violation 
if it contains at least one forbidden string 
or at least one substring matching a forbidden regex pattern.
The violation rates of unconstrained models 
are reported in the main throughput tables.
For the constrained methods, both rejection sampling and \ours{} achieve zero
violation rate under the evaluated constraints.
We measure generation quality using perplexity, 
but perplexity remains at a similar level across methods, 
including the baselines.
Absolute throughput values for base models, rejection sampling, GUARD, and \ours{}
as well as perplexity results, and violation rates for all methods,
are provided in Appendix~\ref{NCO:app:sec:additional_experimental_results}.

\begin{table*}[th]
    \centering
    \small
    \caption{
        Comparison of relative throughput (= ratio of tokens-per-sec between method and base) on the profanity suppression task.
        The violation rate of each  model is reported in parentheses. 
        RS denotes rejection sampling and GUARD denotes the evaluated trie-based component of GUARD.
    }
    \setlength{\tabcolsep}{4.5pt}
    \label{NCO:main:tab:profanity_main_results}
    \begin{tabular}{ccrrrrrrr}
        \toprule[1.2pt]
        \multirow{2}{*}[-0.5ex]{Model} & \multirow{2}{*}[-0.5ex]{Method} & \multicolumn{7}{c}{Batch Size} \\
        \cmidrule{3-9}
        \multicolumn{2}{c}{} & \multicolumn{1}{c}{1} & \multicolumn{1}{c}{2} & \multicolumn{1}{c}{4} & \multicolumn{1}{c}{8} & \multicolumn{1}{c}{16} & \multicolumn{1}{c}{32} & \multicolumn{1}{c}{64} \\
        \midrule
        \multirow{3}{*}{\makecell[c]{Llama~2~7B\\(14.8\%)}}
        & RS & 99.4\tiny{$\pm$0.0} & 98.8\tiny{$\pm$0.2} & 98.1\tiny{$\pm$0.1} & 97.5\tiny{$\pm$0.1} & 96.7\tiny{$\pm$0.2} & 94.7\tiny{$\pm$0.5} & 93.9\tiny{$\pm$0.6} \\
        & GUARD & 70.1\tiny{$\pm$0.3} & 55.4\tiny{$\pm$0.9} & 36.7\tiny{$\pm$1.1} & 24.7\tiny{$\pm$0.5} & 16.3\tiny{$\pm$0.5} & 13.0\tiny{$\pm$0.3} & 10.2\tiny{$\pm$0.2} \\
        & \CCG\ours{} & \CCG99.9\tiny{$\pm$0.1} & \CCG99.5\tiny{$\pm$0.5} & \CCG99.9\tiny{$\pm$0.2} & \CCG100.2\tiny{$\pm$0.2} & \CCG100.2\tiny{$\pm$0.3} & \CCG99.9\tiny{$\pm$0.4} & \CCG100.3\tiny{$\pm$0.1} \\
        \midrule
        \multirow{3}{*}{\makecell[c]{Llama~3.1~8B\\(15.0\%)}}
        & RS & 99.4\tiny{$\pm$0.1} & 98.7\tiny{$\pm$0.6} & 98.5\tiny{$\pm$0.5} & 98.2\tiny{$\pm$0.4} & 97.3\tiny{$\pm$0.2} & 96.3\tiny{$\pm$0.6} & 95.0\tiny{$\pm$0.7} \\
        & GUARD & 38.6\tiny{$\pm$0.4} & 23.9\tiny{$\pm$0.3} & 13.8\tiny{$\pm$0.3} & 8.7\tiny{$\pm$0.2} & 5.3\tiny{$\pm$0.2} & 4.0\tiny{$\pm$0.1} & 3.0\tiny{$\pm$0.1} \\
        & \CCG\ours{} & \CCG99.9\tiny{$\pm$0.0} & \CCG99.6\tiny{$\pm$0.7} & \CCG99.8\tiny{$\pm$0.3} & \CCG100.1\tiny{$\pm$0.3} & \CCG100.4\tiny{$\pm$0.5} & \CCG100.3\tiny{$\pm$0.2} & \CCG100.3\tiny{$\pm$0.5} \\
        \midrule
        \multirow{3}{*}{\makecell[c]{Qwen~2.5~7B\\(17.0\%)}}
        & RS & 100.5\tiny{$\pm$1.2} & 97.6\tiny{$\pm$0.7} & 98.5\tiny{$\pm$1.1} & 97.0\tiny{$\pm$0.3} & 95.5\tiny{$\pm$0.2} & 93.5\tiny{$\pm$0.2} & 91.5\tiny{$\pm$0.7} \\
        & GUARD & 35.6\tiny{$\pm$0.4} & 19.1\tiny{$\pm$0.3} & 11.0\tiny{$\pm$0.1} & 6.5\tiny{$\pm$0.2} & 3.6\tiny{$\pm$0.2} & 2.2\tiny{$\pm$0.0} & 1.5\tiny{$\pm$0.2} \\
        & \CCG\ours{} & \CCG100.7\tiny{$\pm$1.3} & \CCG100.2\tiny{$\pm$0.8} & \CCG101.4\tiny{$\pm$0.0} & \CCG100.8\tiny{$\pm$0.4} & \CCG100.6\tiny{$\pm$0.5} & \CCG101.3\tiny{$\pm$1.3} & \CCG100.8\tiny{$\pm$0.7} \\
        \midrule
        \multirow{3}{*}{\makecell[c]{Falcon~7B\\(11.6\%)}}
        & RS & 100.9\tiny{$\pm$1.4} & 98.3\tiny{$\pm$0.7} & 96.5\tiny{$\pm$1.4} & 97.7\tiny{$\pm$0.6} & 96.0\tiny{$\pm$0.1} & 93.7\tiny{$\pm$1.1} & 90.3\tiny{$\pm$4.2} \\
        & GUARD & 53.5\tiny{$\pm$0.3} & 35.2\tiny{$\pm$0.3} & 20.1\tiny{$\pm$0.6} & 12.8\tiny{$\pm$0.2} & 7.4\tiny{$\pm$0.3} & 4.9\tiny{$\pm$0.1} & 3.3\tiny{$\pm$0.1} \\
        & \CCG\ours{} & \CCG101.0\tiny{$\pm$1.5} & \CCG99.6\tiny{$\pm$0.7} & \CCG100.6\tiny{$\pm$0.6} & \CCG100.7\tiny{$\pm$0.2} & \CCG100.6\tiny{$\pm$0.4} & \CCG100.0\tiny{$\pm$1.4} & \CCG 98.7\tiny{$\pm$3.8} \\
        \midrule
        \multirow{3}{*}{\makecell[c]{Falcon~3~7B\\(14.6\%)}} 
        & RS & 100.7\tiny{$\pm$1.3} & 98.3\tiny{$\pm$0.7} & 98.1\tiny{$\pm$0.8} & 97.0\tiny{$\pm$0.4} & 96.1\tiny{$\pm$1.3} & 95.0\tiny{$\pm$2.6} & 92.9\tiny{$\pm$2.7} \\
        & GUARD & 36.5\tiny{$\pm$0.3} & 20.8\tiny{$\pm$0.4} & 12.2\tiny{$\pm$0.5} & 7.0\tiny{$\pm$0.4}& 4.0\tiny{$\pm$0.2} & 2.5\tiny{$\pm$0.2} & 1.8\tiny{$\pm$0.1} \\
        & \CCG\ours{} & \CCG100.8\tiny{$\pm$1.5} & \CCG100.4\tiny{$\pm$0.7} & \CCG100.6\tiny{$\pm$1.1} & \CCG101.4\tiny{$\pm$0.9} & \CCG101.8\tiny{$\pm$1.4} & \CCG102.0\tiny{$\pm$2.2} & \CCG103.6\tiny{$\pm$4.2} \\
        \midrule
        \multirow{3}{*}{\makecell[c]{Phi-4~14B\\(16.6\%)}} 
        & RS & 100.5\tiny{$\pm$1.0} & 98.5\tiny{$\pm$0.9} & 99.3\tiny{$\pm$0.5} & 99.1\tiny{$\pm$1.0} & 99.1\tiny{$\pm$0.4} & 97.4\tiny{$\pm$2.2} & 96.3\tiny{$\pm$0.7} \\
        & GUARD & 61.2\tiny{$\pm$0.6} & 37.9\tiny{$\pm$0.6} & 24.4\tiny{$\pm$0.3} & 15.1\tiny{$\pm$0.3} & 9.2\tiny{$\pm$0.4} & 5.9\tiny{$\pm$0.5} & 4.3\tiny{$\pm$0.1} \\
        & \CCG\ours{} & \CCG100.3\tiny{$\pm$0.9} & \CCG99.3\tiny{$\pm$1.1} & \CCG100.5\tiny{$\pm$3.1} & \CCG101.4\tiny{$\pm$1.4} & \CCG101.7\tiny{$\pm$0.7} & \CCG101.5\tiny{$\pm$1.2} & \CCG99.9\tiny{$\pm$0.2} \\
        \bottomrule[1.2pt]
    \end{tabular}
\end{table*}

\subsection{Main Results}
\label{NCO:main:ssec:main_results}

\paragraph{Profanity suppression.}
Table~\ref{NCO:main:tab:profanity_main_results} reports relative throughput on the
profanity suppression task.
The violation rate shown below each model name indicates that the
unconstrained base models generate forbidden strings at non-negligible rates.
All constrained methods reduce violations under the evaluated string
constraints, but they differ in decoding efficiency.
As the batch size increases, rejection sampling and GUARD show noticeable
throughput degradation.
This trend reflects their sequential constraint handling procedures.
Rejection sampling may repeatedly sample and check candidate tokens, and GUARD
relies on trie-based suppression 
that becomes increasingly expensive under batched generation.
In contrast, \ours{} maintains compact Aho-Corasick states 
and computes token-level masks using precomputed transition tables.
As a result, \ours{} preserves relative throughput 
close to the unconstrained base model across batch sizes.
The advantage of \ours{} is most visible in large batch settings.
Compared with rejection sampling, 
\ours{} achieves roughly five to six percentage points higher relative throughput.
These results show that \ours{} can enforce a large set of finite hard constraints 
with little overhead while retaining the benefits of batched decoding.

\paragraph{PII suppression.}
Table~\ref{NCO:main:tab:pii_main_results} reports relative throughput 
on the PII suppression task.
The base model violation rates shown below the model names are high
because the evaluated models are trained on the Enron Email Dataset.
The PII task is more challenging than profanity suppression because the
constraints are specified as regex patterns rather than finite hard constraints.
GUARD handles only a finite-string approximation of the original regex
constraints, but still obtains little benefit from batching.
Its throughput drops substantially at larger batch sizes 
due to sequential trie-based constraint handling.
Rejection sampling can enforce these constraints 
by checking sampled continuations, 
but its sequential candidate rejection procedure again becomes less efficient 
as the batch size grows.
In contrast, \ours{} maintains stable relative throughput 
across batch sizes and remains close to the unconstrained base model 
in most settings.
This behavior supports the main design goal of \ours{}.
The method moves constraint processing into parallel mask computation 
and precomputed automaton transitions, 
which makes it suitable for batched deployment 
with multiple negative regex constraints.

\begin{table*}[ht]
    \centering
    \caption{
        Relative throughput on the PII suppression task.
        Each value is normalized 
        by the throughput of the corresponding unconstrained base model 
        at the same batch size.
        The violation rate of each unconstrained base model is shown below the model name.
    }
    \label{NCO:main:tab:pii_main_results}
    \begin{tabular}{ccrrrrr}
        \toprule[1.2pt]
        \multirow{2}{*}[-0.5ex]{Model} & \multirow{2}{*}[-0.5ex]{Method} & \multicolumn{5}{c}{Batch Size} \\
        \cmidrule{3-7}
        \multicolumn{2}{c}{} & \multicolumn{1}{c}{1} & \multicolumn{1}{c}{2} & \multicolumn{1}{c}{4} & \multicolumn{1}{c}{8} & \multicolumn{1}{c}{16} \\
        \midrule
        \multirow{3}{*}{\makecell[c]{GPT-J~6B\\(87.0\%)}}
        & RS & 99.2\footnotesize{$\pm$0.1} & 98.6\footnotesize{$\pm$0.2} & 98.0\footnotesize{$\pm$0.3} & 97.3\footnotesize{$\pm$0.2} & 95.7\footnotesize{$\pm$0.3} \\
        & GUARD & 59.1\footnotesize{$\pm$0.2} & 48.2\footnotesize{$\pm$0.2} & 31.5\footnotesize{$\pm$0.1} & 22.1\footnotesize{$\pm$0.1} & 16.8\footnotesize{$\pm$0.1}\\
        & \CCG\ours{} & \CCG99.3\footnotesize{$\pm$0.1} & \CCG99.3\footnotesize{$\pm$0.2} & \CCG99.1\footnotesize{$\pm$0.5} & \CCG99.3\footnotesize{$\pm$0.3} & \CCG99.4\footnotesize{$\pm$0.2} \\
        \midrule
        \multirow{3}{*}{\makecell[c]{GPT-Neo~2.7B\\(81.8\%)}}
        & RS & 100.7\footnotesize{$\pm$3.7} & 96.9\footnotesize{$\pm$0.4} & 96.3\footnotesize{$\pm$0.2} & 95.0\footnotesize{$\pm$0.1} & 93.0\footnotesize{$\pm$0.3} \\
        & GUARD & 40.9\footnotesize{$\pm$2.0} & 29.6\footnotesize{$\pm$0.1} & 19.4\footnotesize{$\pm$0.0} & 13.4\footnotesize{$\pm$0.1} & 10.0\footnotesize{$\pm$0.0} \\
        & \CCG\ours{} & \CCG100.5\footnotesize{$\pm$3.7} & \CCG98.5\footnotesize{$\pm$0.3} & \CCG99.0\footnotesize{$\pm$0.2} & \CCG98.9\footnotesize{$\pm$0.2} & \CCG99.0\footnotesize{$\pm$0.2}\\
        \midrule
        \multirow{3}{*}{\makecell[c]{Pythia~6.9B\\(82.2\%)}}
        & RS & 100.0\footnotesize{$\pm$1.1} & 97.5\footnotesize{$\pm$0.3} & 96.6\footnotesize{$\pm$2.3} & 98.7\footnotesize{$\pm$1.4} & 94.4\footnotesize{$\pm$1.6} \\
        & GUARD & 63.6\footnotesize{$\pm$0.2} & 45.8\footnotesize{$\pm$0.8} & 31.1\footnotesize{$\pm$0.2} & 25.5\footnotesize{$\pm$0.3} & 20.5\footnotesize{$\pm$0.6} \\
        & \CCG\ours{} & \CCG100.0\footnotesize{$\pm$1.1} & \CCG97.9\footnotesize{$\pm$0.9} & \CCG99.3\footnotesize{$\pm$0.4} & \CCG99.8\footnotesize{$\pm$2.5} & \CCG102.2\footnotesize{$\pm$2.1} \\
        \bottomrule[1.2pt]
    \end{tabular}
\end{table*}

\subsection{Ablation Study}
\label{NCO:main:ssec:ablation}
We conduct ablation studies 
to isolate the contribution of each acceleration component in \ours{}. 
For finite hard constraints, we compare the full method with variants 
that remove Aho-Corasick failure links or BPE-based transition precomputation.
Removing failure links reduces the method to plain trie-style matching 
and increases decoding overhead under many forbidden strings. 
Removing BPE-based precomputation mainly affects the one-time setup cost, 
since token transitions must be computed with less reuse.
For regex constraints, we ablate the suffix-state precomputation 
and GPU-based mask aggregation. 
The suffix-state table accounts for matches 
that begin inside a newly appended token, 
which is necessary for substring-level regex avoidance under subword tokenization. 
GPU-based aggregation combines masks from many active DFA states in parallel 
and reduces the runtime cost of enforcing multiple regex constraints.

Table~\ref{NCO:main:tab:ablation} reports the ablation results. 
Implementation details for each ablation variant are provided 
in Appendix~\ref{NCO:app:ssec:ablation_implementation_details}.
The finite hard constraint ablations separate runtime and preprocessing effects. 
Failure links are important during decoding 
because they avoid repeated suffix searches under many forbidden strings. 
BPE-based transition precomputation mainly reduces setup time 
by reusing transitions of previously merged tokens.
The regex constraint ablations show a different bottleneck. 
GPU-based mask aggregation reduces the cost of combining 
many active DFA states. 
Suffix-reachable state precomputation plays a complementary role. 
It moves the ambiguity of matches starting inside a subword token 
from decoding time to setup time 
by storing suffix-reachable states in advance. 
Together, these results indicate that the efficiency gain of \ours{} 
does not come from a single optimization.
The acceleration components reduce different costs 
in the constraint  enforcement pipeline, 
which explains why their benefits are complementary.
This trade-off is favorable when the same constraint set is reused 
across many generations, as in typical safety-filter deployments.

\begin{table*}[t!]
\centering
\caption{
Ablation study of \ours{}.
We report absolute throughput and precomputation time at batch size $1$. Across all ablation variants, the violation rate remained at $0$.
}
\label{NCO:main:tab:ablation}
\begin{tabular}{lllrr}
\toprule[1.2pt]
Constraint Type & Model & Method & Throughput~($\uparrow$) & Pre. Time~($\downarrow$) \\
\midrule
\multirow{6}{*}{Lexicon} & \multirow{3}{*}{Llama~2~7B} & \CCG\ours{} & \CCG50.29\small{$\pm$0.04} & \CCG29.66\small{$\pm$0.44} \\
& & ~~w/o AC Trie & 35.29\small{$\pm$0.14} & 0.00\small{$\pm$0.00} \\
& & ~~w/o BPE Precomp. & 50.31\small{$\pm$0.02} & 47.03\small{$\pm$1.09} \\
\cmidrule{2-5}
& \multirow{3}{*}{Llama~3.1~8B} & \CCG\ours{} & \CCG44.94\small{$\pm$0.03} & \CCG126.95\small{$\pm$0.78} \\
& & ~~w/o AC Trie & 17.34\small{$\pm$0.15} & 0.00\small{$\pm$0.00} \\
& & ~~w/o BPE Precomp. & 44.95\small{$\pm$0.03} & 214.49\small{$\pm$1.61} \\
\midrule
\multirow{6}{*}{Regexes} & \multirow{3}{*}{GPT-J~6B} & \CCG\ours{} & \CCG53.87\small{$\pm$0.04} & \CCG1.91\small{$\pm$0.04} \\
& & ~~w/o Suffix Map & 53.68\small{$\pm$0.05} & 0.76\small{$\pm$0.07} \\
& & ~~w/o MatMul & 53.24\small{$\pm$0.09} & 1.87\small{$\pm$0.00} \\
\cmidrule{2-5}
& \multirow{3}{*}{GPT-Neo~2.7B} & \CCG\ours{} & \CCG106.66\small{$\pm$0.08} & \CCG1.91\small{$\pm$0.04} \\
& & ~~w/o Suffix Map & 105.96\small{$\pm$0.11} & 0.83\small{$\pm$0.00} \\
& & ~~w/o MatMul & 103.81\small{$\pm$0.08} & 1.91\small{$\pm$0.06} \\
\bottomrule[1.2pt]
\end{tabular}
\end{table*}

\paragraph{Additional analyses.}

We include additional experiments that complement the main results.
Absolute throughput measurements 
in Appendix~\ref{NCO:app:ssec:absolute_throughput_main_tasks}
show that the relative trends are not artifacts of normalization 
across models and batch sizes.
Quality and violation evaluations 
in Appendix~\ref{NCO:app:ssec:quality_and_violation_analysis}
confirm that hard masking with \ours{} removes constraint violations 
while keeping perplexity close to the base model.
Sensitivity analyses in Appendix~\ref{NCO:app:ssec:sensitivity_analysis}
further show that the decoding overhead remains small 
as the forbidden lexicon grows, strings become longer, or regex DFAs become larger.
We also evaluate a soft constraint mode 
in Appendix~\ref{NCO:app:ssec:soft_negative_constraints},
where invalid continuations receive finite logit penalties 
rather than being fully masked.
Across these additional settings, \ours{} consistently preserves the main 
efficiency advantage observed in the primary experiments.
The quality results also show that exact constraint enforcement does not require 
a substantial degradation in language modeling behavior.
The soft constraint results further suggest that the same framework can be used 
when applications prefer suppression over hard blocking.
These results indicate that \ours{} supports both strict enforcement 
and probabilistic suppression through the same online constraint states.

\section{Conclusion}

We present \ours{}, a decoding-time plug-in 
for preventing forbidden strings and regex patterns from appearing 
as substrings in LLM outputs. 
\ours{} avoids the state explosion of a global avoidance automaton 
by maintaining compact online matching states. 
It uses an Aho-Corasick trie for finite hard constraints 
and parallel DFA simulation for regex constraints, 
with token-level precomputation to reduce repeated decoding-time work.
Across profanity suppression and PII suppression tasks, 
\ours{} achieves zero constraint violations 
while preserving throughput close to unconstrained decoding. 
The method remains compatible with standard decoding strategies
because it intervenes only at the logit level. 
The soft logit setting also allows the same constraint representation 
to provide graded suppression instead of only hard blocking.
While this work is limited to explicit finite hard constraints and regex constraints,
future work may extend negative constraint enforcement to richer formal
specifications, such as grammar-level constraints. 
Dynamic constraint updates are another important direction, 
especially for deployment settings where finite hard constraints change frequently.

\subsection*{AI use statement}

In this work, we did not use generative AI tools for any tasks requiring disclosure.
Additionally, we used generative AI tools to aid and polish the writing of this paper,
including improving clarity, grammar, and presentation.
All AI-assisted text was reviewed and revised by the authors to ensure accuracy
and consistency with the underlying research.
We take responsibility for the final content of this work,
including text, claims or artifacts produced with the aid of generative AI.

\bibliography{iclr2027_conference}
\bibliographystyle{iclr2027_conference}

%%%%%%%%%%%%%%%%%%%%%%%%%%%%%%%%%%%%%%%%%%%%%%%%%%%%%%%%%%%%%%%%%%%%%%%%%%%%%%%
% APPENDIX
%%%%%%%%%%%%%%%%%%%%%%%%%%%%%%%%%%%%%%%%%%%%%%%%%%%%%%%%%%%%%%%%%%%%%%%%%%%%%%%

\appendix
\section{Preliminaries}\label{NCO:app:sec:preliminaries}
We briefly review the background relevant to our method, 
including regular languages, BPE tokenization, 
the closure properties of regular languages under intersection and complement, 
and the Aho-Corasick algorithm.

\subsection{Regular Languages}
A regular language~\citep{Sipser1997} is a set of strings 
over a finite alphabet that can be described using operations 
such as concatenation, union, and the Kleene star. 
Regular languages admit two standard representations, 
regular expressions and finite automata, 
also known as finite state machines~(FSMs).
These two formalisms are equivalent in expressive power, 
so the class of languages representable by regular expressions is exactly 
the class of languages recognizable by finite automata.
In this paper, we use the term regular expression in the formal language sense,
excluding non-regular extensions that appear in some practical regex engines.
A finite automaton is formally defined as a 5-tuple $(Q, \Sigma, \delta, s, F)$,
where $Q$ is a finite set of states, 
$\Sigma$ is a finite alphabet, 
$s \in Q$ is the start state, 
$F \subseteq Q$ is the set of accepting states, 
and $\delta: Q \times \Sigma \to 2^Q$ is the transition function. 
Throughout this paper, 
regex constraints are restricted to this formal setting, 
which ensures convertibility to DFA, 
consistent with prior work~\citep{DeutschUR2019,KooLH2024}.

\subsection{Byte-pair Encoding}
Byte-pair encoding~(BPE) is a widely used subword tokenization algorithm 
for language models. 
BPE tokenizes text into a sequence of tokens 
by iteratively applying learned merge rules 
that combine frequent adjacent symbol pairs. 
Starting from a base alphabet of symbols, 
BPE training produces an ordered list of merge operations. 
The tokenization procedure applies these rules 
to obtain a segmentation of an input string into subword tokens.
The key property of BPE that we exploit in this work is 
its hierarchical composition structure. 
Each token $w \in \mathcal{V}$ can be viewed as 
the result of merging two shorter symbols, written as
\begin{equation}
w = u \cdot v,
\end{equation}
where $u$ and $v$ are either base symbols or previously merged tokens, 
and $\cdot$ denotes string concatenation. 
As a result, each token can be associated with an implicit binary merge tree.
This structure makes it possible to compute token-level quantities recursively 
from their left and right parts. 
When a quantity depends only on the string represented by a token 
and can be composed from its two children, 
intermediate results can be reused across tokens 
that share substructures, enabling efficient precomputation.

\subsection{Intersection and Complement of Regular Languages}
Regular languages are closed under Boolean operations 
including intersection and complement.
These closure properties are useful for expressing complex constraints, 
but the corresponding automaton constructions can be expensive 
when they are carried out explicitly.
We first recall the intersection construction for DFAs.
Given two DFAs
\[
    A_1 = (Q_1,\Sigma,\delta_1,s_1,F_1)
    \qquad
    \text{and}
    \qquad
    A_2 = (Q_2,\Sigma,\delta_2,s_2,F_2),
\]
a DFA for $L(A_1)\cap L(A_2)$ can be obtained by the product construction.
Its state space is $Q_1\times Q_2$, 
and therefore has size at most $|Q_1||Q_2|$.
For $l$ constraints, the same construction yields a product state space of size
\[
    |Q_1|\cdot |Q_2|\cdots |Q_l|.
\]
Thus, even if each individual automaton is small, 
explicitly combining many constraints can lead to exponential growth 
in the number of constraints.

Complementation is simple once a complete DFA has been constructed.
For a complete DFA
\[
    A=(Q,\Sigma,\delta,s,F),
\]
where $\delta(q,c)$ is defined for every $q\in Q$ and every $c\in\Sigma$,
the complement is obtained by replacing $F$ with $Q\setminus F$.
Thus, complementing a complete DFA does not itself 
increase the number of states.
However, this operation is directly valid only for complete DFAs.
If the automaton is partial, missing transitions must first be completed,
typically by adding a sink state.
If the constraint is given as a regular expression or NFA, 
determinization may also be required before complementing it.

The state growth becomes more severe for negative substring constraints.
For a forbidden language $L(A)$, 
invalid outputs are strings that contain some string in $L(A)$ as a substring.
This language is
\[
    \Sigma^* L(A) \Sigma^*.
\]
The leading $\Sigma^*$ means 
that a match may start at any position in the generated text.
A nondeterministic automaton can represent this compactly 
by starting a possible match at each position.
A deterministic automaton, however, must summarize all currently possible
partial matches in a single state.

For example, let $\Sigma=\{a,b\}$ and consider the forbidden pattern
\[
    a\Sigma^n a.
\]
The substring occurrence language is
\[
    \Sigma^* a\Sigma^n a \Sigma^*.
\]
An NFA can recognize this language with $O(n)$ states 
by guessing the first $a$, reading exactly $n$ arbitrary symbols, 
and then checking whether the next symbol is $a$.
In contrast, a DFA must remember enough information about the recent history 
to know whether an $a$ occurred exactly $n+1$ positions before the current symbol.
More generally, it must track many possible match starts at once, 
which can be viewed as a subset of active NFA states.
This subset representation can require exponentially many states 
in the worst case.

Therefore, a direct construction for multiple negative regex constraints
combines two sources of growth.
Each substring occurrence language $\Sigma^*L(A_i)\Sigma^*$ may already 
require a large deterministic automaton.
Then enforcing all constraints simultaneously requires 
intersecting their complements.
This is why relying only on the closure properties of regular languages 
does not yield a practical decoding method.

\subsection{Aho-Corasick Algorithm}
The Aho-Corasick algorithm is an extension of the Knuth-Morris-Pratt (KMP)
algorithm, designed for efficient multi-pattern matching against a single input
text. This algorithm constructs a trie from a set of patterns and utilizes a
failure function, the core mechanism of KMP, to achieve efficient pattern
matching. The failure function indicates which node to transition to upon a
character mismatch, allowing the matching process to continue. Its time and
space complexity are linear with respect to the total length of the text and the
patterns. If the failure transitions are precomputed and integrated into the
transition function, the resulting structure is fundamentally equivalent to a
Deterministic Finite Automaton~(DFA).

\section{Detailed Algorithms}\label{NCO:app:sec:detailed_algorithms}
We provide detailed pseudocode for the algorithms used in \ours{}.
We first present naive implementations 
that compute token-level transition and blocking tables 
by directly simulating each token string. 
We then describe the BPE-aware precomputation used in our implementation, 
which reuses the merge structure of BPE tokenizers 
to reduce redundant computation. 
Finally, we present the matrix multiplication trick 
used to compute masks efficiently during decoding.
Throughout this section, $\Gamma$ denotes the tokenizer vocabulary, 
and each token $w\in\Gamma$ is treated as a string over $\Sigma$.
When discussing BPE merges, we write a merged token as $w=u\cdot v$, 
where $u$ and $v$ are previously constructed tokens.
We use \texttt{EOS}$\in\Gamma$ to denote the end-of-sequence token 
and treat it as a control token rather than a symbol sequence over $\Sigma$.
For a precomputed blocking mask $B$, 
the value $B[q,w]=1$ means that token $w$ must be blocked 
when the current constraint state is $q$.

\subsection{Naive Implementation}

\subsubsection{Finite Hard Constraints}

\paragraph{Precomputation.}
For forbidden hard constraints, 
the naive implementation first constructs an Aho-Corasick trie 
over the forbidden string set $P$.
Each trie state represents a matched prefix of some forbidden string, 
and the failure links allow the matcher to fall back 
to the longest suffix that is also a trie prefix. 
After the trie is constructed, we precompute two token-level tables. 
The transition table $\Delta[q,w]$ stores the trie state 
reached after reading token $w$ from state $q$.
The blocking mask $B[q,w]$ records 
whether reading $w$ from $q$ visits a forbidden state, 
in which case appending $w$ would create a forbidden substring.
This table directly characterizes 
whether $w$ belongs to the valid candidate set $\Gamma_P(Y_t)$ 
for a prefix whose current trie state is $q$.

\begin{algorithm}
\caption{Naive Precomputation for Finite Hard Constraints}
\label{NCO:app:alg:naive_string_precompute}
\begin{algorithmic}[1]
\Require Forbidden strings $P\subseteq\Sigma^*$ and tokenizer vocabulary $\Gamma$
\Ensure Token transition table $\Delta$, blocking mask $B$, and root state $q_{\mathrm{root}}$

\Lcomment{Build AC trie}
\State Build an Aho--Corasick trie $T$ from $P$
\State Let $Q$ be the set of states in $T$, and let $q_{\mathrm{root}}$ be the root state
\State Compute failure links $\mathrm{fail}$ for all states in $Q$
\State Mark state $q$ as forbidden if $q$ corresponds to a complete forbidden string
\State Propagate forbidden marks so that if $\mathrm{fail}(q)$ is forbidden, then $q$ is also forbidden

\Lcomment{Precompute token-level transitions and blocking masks}
\ForAll{$q\in Q$}
    \ForAll{$w\in\Gamma$}
        \State $q'\gets q$
        \State $B[q,w]\gets 0$
        \ForAll{characters $a$ in token string $w$}
            \While{$q'\neq q_{\mathrm{root}}$ and $q'$ has no transition labeled $a$}
                \State $q'\gets \mathrm{fail}[q']$
            \EndWhile
            \If{$q'$ has a transition labeled $a$}
                \State $q'\gets$ the child of $q'$ labeled by $a$
            \Else
                \State $q'\gets q_{\mathrm{root}}$
            \EndIf
            \If{$q'$ is forbidden}
                \State $B[q,w]\gets 1$
            \EndIf
        \EndFor
        \State $\Delta[q,w]\gets q'$
    \EndFor
\EndFor
\State Mark $B[q,\texttt{EOS}]\gets 0$ for all $q\in Q$
\State \Return $\Delta,B,q_{\mathrm{root}}$
\end{algorithmic}
\end{algorithm}

\paragraph{Decoding.}
During decoding, the AC state summarizes all information needed 
to determine whether a candidate token would create a forbidden substring.
Given the current prefix $Y_t$, 
let $q_t$ be the AC state obtained after reading $Y_t$. 
For each candidate token $w$, the precomputed mask $B[q_t,w]$ indicates 
whether $Y_t\cdot w$ would violate the hard constraint. 
\ours{} therefore masks exactly those invalid tokens 
and then applies an arbitrary decoding rule, 
such as greedy decoding, sampling, or beam search, over the filtered logits.
After a token is selected, the AC state is updated by a single table lookup.

\begin{algorithm}
\caption{Decoding with Finite Hard Constraints}
\label{NCO:app:alg:string_decode}
\begin{algorithmic}[1]
\Require Model $M_\theta$, input string $X$, tokenizer vocabulary $\Gamma$, decoding rule $\mathsf{Decode}$
\Require Token transition table $\Delta$, blocking mask $B$, and root state $q_{\mathrm{root}}$
\Ensure Output string $Y$ satisfying $\mathrm{Valid}_{P}$

\State $Y_0\gets\epsilon$
\State $q_0\gets q_{\mathrm{root}}$

\For{$t=0,1,\ldots$ until termination}
    \State $\ell_t\gets$ next-token logits from $M_\theta$ conditioned on $(X,Y_t)$
    \ForAll{$w\in\Gamma$}
        \If{$B[q_t,w]=1$}
            \State $\ell_t(w)\gets -\infty$ \Rcomment{$w\notin\Gamma_P(Y_t)$}
        \EndIf
    \EndFor
    \State $y_{t+1}\gets \mathsf{Decode}(\ell_t)$
    \State $Y_{t+1}\gets Y_t\cdot y_{t+1}$
    \State $q_{t+1}\gets \Delta[q_t,y_{t+1}]$
\EndFor
\State \Return $Y$
\end{algorithmic}
\end{algorithm}

\subsubsection{Regex Constraints}

\paragraph{Precomputation.}
For forbidden regex constraints, 
each DFA $A_i\in D$ recognizes a forbidden language $L(A_i)$. 
As in the finite hard constraint case, 
a forbidden match may start at any position in the generated output. 
However, unlike finite hard constraints handled by an AC trie, 
a general DFA does not provide failure links 
that summarize all relevant suffix matches as a single state. 
We therefore maintain the set of DFA states reachable 
by reading suffixes of the current output prefix.
Strictly speaking, a DFA has a defined transition 
for every state and every alphabet symbol. 
In our implementation, we use an equivalent partial representation 
that omits the non-accepting sink state and all transitions into it. 
Once a run enters this sink state, 
it can never reach an accepting state 
and therefore can never contribute to a future violation. 
We denote this omitted sink state by $\bot$ in the algorithms. 
Thus, when $\delta_i(q,a)$ is undefined,
we set the current state to $\bot$ and stop simulating that run.
The naive precomputation builds three token-level tables for each DFA. 
The transition table $\Delta_i[q,w]$ stores the state reached 
after reading token $w$ from DFA state $q$, 
or $\bot$ if the run enters the omitted sink state. 
The blocking mask $B_i[q,w]$ records 
whether reading $w$ from $q$ reaches an accepting state. 
Finally, the suffix-reachable set $R_i[w]$ stores the states
reachable from the initial state $s_i$ by reading any suffix of token $w$. 
This last table is needed 
because a new forbidden match may begin inside the newly appended token.

\begin{algorithm}
\caption{Naive Precomputation for Forbidden Regex Constraints}
\label{NCO:app:alg:naive_dfa_precompute}
\begin{algorithmic}[1]
\Require Forbidden DFAs $D=\{A_1,\ldots,A_m\}$ and tokenizer vocabulary $\Gamma$
\Ensure Token transitions $\{\Delta_i\}_{i=1}^{m}$, blocking masks $\{B_i\}_{i=1}^{m}$, and suffix-reachable sets $\{R_i\}_{i=1}^{m}$

\ForAll{$A_i=(Q_i,\Sigma,\delta_i,s_i,F_i)\in D$}
    \Lcomment{Precompute token transitions and blocking masks}
    \ForAll{$q\in Q_i$}
        \ForAll{$w\in\Gamma$}
            \State{Initialize $q'$ with $q$ and $B_i[q,w]$ with $0$.}
            \State{Mark $B_i[q,w]$ as $1$ if $q'$ is an accepting state.}
            \ForAll{characters $a$ in token string $w$}
                \State{Take the transition obtained by reading $a$ from $q'$.}
                \State{Set $q'$ to $\bot$ and stop scanning if no such transition exists.}
                \State{Mark $B_i[q,w]$ as $1$ if $q'$ is an accepting state.}
            \EndFor
            \State{Record the transition $\Delta_i[q,w]$ as $q'$.}
        \EndFor
    \EndFor
    \Lcomment{Precompute states reachable from suffixes of each token}
    \ForAll{$w\in\Gamma$} 
        \State{Initialize $R_i[w]$ as an empty set.}
        \ForAll{suffixes $u$ of token string $w$}
            \State Initialize $q'$ with $s_i$ and $\mathrm{accepted}$ with whether $s_i$ is accepting.
            \ForAll{characters $a$ in $u$}
                \State{Take the transition obtained by reading $a$ from $q'$.}
                \State{Set $q'$ to $\bot$ and stop scanning $u$ if no such transition exists.}
                \State{Mark $\mathrm{accepted}$ as true if the updated $q'$ is an accepting state.}
            \EndFor
            \If{$q'\neq\bot$}
                \State Add $q'$ to $R_i[w]$.
            \EndIf
            \If{$\mathrm{accepted}$ is true}
                \State{Mark $B_i[q,w]$ as $1$ for all states $q\in Q_i$.}
            \EndIf
        \EndFor
    \EndFor
    \State{Mark $B_i[q,\texttt{EOS}]$ to $0$ for all states $q\in Q_i$.}
\EndFor
\State \Return $\{\Delta_i,B_i,R_i\}_{i=1}^{m}$
\end{algorithmic}
\end{algorithm}

\paragraph{Decoding.}
During decoding, each DFA $A_i$ maintains an active state set $S_{i,t}$ 
after reading the current output prefix $Y_t$. 
Intuitively, $S_{i,t}$ represents all partial matches of 
the forbidden regex constraints 
that could be in progress at the end of $Y_t$. 
Equivalently, it contains the DFA states 
reached by running $A_i$ on suffixes of $Y_t$. 
We also keep the initial state $s_i$ active, 
because a forbidden match may start at the next generated token.
Given this active state set, a candidate token $w$ is invalid 
if reading $w$ from any active state reaches an accepting state, 
which is exactly recorded by the precomputed blocking mask $B_i[q,w]$. 
After a token is selected, the active state set is updated 
by advancing all previous active states through the selected token 
and adding states reachable by suffixes of the selected token. 
This update accounts for matches that started before the token
as well as matches that start inside the token.

\begin{algorithm}
\caption{Decoding with Forbidden Regex Constraints}
\label{NCO:app:alg:dfa_decode}
\begin{algorithmic}[1]
\Require Model $M_\theta$, input string $X$, tokenizer vocabulary $\Gamma$, decoding rule $\mathsf{Decode}$
\Require Token transitions $\{\Delta_i\}_{i=1}^{m}$, blocking masks $\{B_i\}_{i=1}^{m}$, and suffix-reachable sets $\{R_i\}_{i=1}^{m}$
\Ensure Output string $Y$ satisfying $\mathrm{Valid}_{D}$

\State $Y_0\gets\epsilon$
\ForAll{$A_i=(Q_i,\Sigma,\delta_i,s_i,F_i)\in D$}
    \State $S_{i,0}\gets\{s_i\}$
\EndFor

\For{$t=0,1,\ldots$ until termination}
    \State $\ell_t\gets$ next-token logits from $M_\theta$ conditioned on $(X,Y_t)$

    \ForAll{$w\in\Gamma$}
        \If{there exist $i$ and $q\in S_{i,t}$ such that $B_i[q,w]=1$}
            \State $\ell_t(w)\gets -\infty$ \Rcomment{$w\notin\Gamma_D(Y_t)$}
        \EndIf
    \EndFor

    \State $y_{t+1}\gets \mathsf{Decode}(\ell_t)$
    \State $Y_{t+1}\gets Y_t\cdot y_{t+1}$

    \ForAll{$A_i\in D$}
        \State $S_{i,t+1}\gets R_i[y_{t+1}]\cup\{s_i\}$
        \ForAll{$q\in S_{i,t}$}
            \If{$\Delta_i[q,y_{t+1}]\neq\bot$}
                \State $S_{i,t+1}\gets S_{i,t+1}\cup\{\Delta_i[q,y_{t+1}]\}$
            \EndIf
        \EndFor
    \EndFor
\EndFor
\State \Return $Y$
\end{algorithmic}
\end{algorithm}

\subsection{BPE Merging}

The naive precomputation computes token-level tables 
by directly simulating each token string from each constraint state. 
As a result, its cost scales linearly with the number of vocabulary tokens 
and also depends on the length of each token string. 
This becomes expensive for modern tokenizers with large vocabularies,
especially because merged BPE tokens can represent long strings.
\ours{} exploits the merge structure of BPE tokenizers to reduce this redundancy. 
Base tokens are handled by direct simulation, 
while each merged token is computed 
by composing the precomputed tables of its two children. 
This allows token-level transitions and blocking masks for merged tokens 
to be obtained by table lookup rather than 
by scanning their full string representations.

\subsubsection{Finite Hard Constraints}

\paragraph{Precomputation.}
For finite hard constraints, the BPE-aware precomputation uses 
the compositional property of AC-trie transitions. 
Suppose a merged token $w$ is formed by
concatenating two previously constructed tokens $u$ and $v$, 
so that $w = u \cdot v$. 
Reading $w$ from state $q$ is equivalent to first reading $u$ from $q$,
reaching an intermediate state $\Delta[q,u]$, 
and then reading $v$ from that intermediate state. 
Thus, the transition for $w$ can be written as
\[
    \Delta[q,w] = \Delta[\Delta[q,u],v].
\]
The blocking decision is compositional as well. 
Token $w$ is blocked from state $q$ 
if reading $u$ from $q$ already creates a forbidden match, 
or if reading $v$ from the intermediate state $\Delta[q,u]$ 
creates a forbidden match. 
Therefore,
\[
    B[q,w] = B[q,u] \lor B[\Delta[q,u],v].
\]
After base tokens are computed by direct AC simulation, 
all merged tokens can be filled by table lookup and composition.

\begin{algorithm}
\caption{BPE-Aware Precomputation for Finite Hard Constraints}
\label{NCO:app:alg:bpe_string_precompute}
\begin{algorithmic}[1]
\Require Forbidden strings $P\subseteq\Sigma^*$, tokenizer vocabulary $\Gamma$, BPE merge rules $\mathcal{R}$
\Ensure Token transition table $\Delta$, blocking mask $B$, and root state $q_{\mathrm{root}}$

\Lcomment{Build AC trie}
\State Build an Aho--Corasick trie $T$ from $P$
\State Let $Q$ be the set of states in $T$, and let $q_{\mathrm{root}}$ be the root state
\State Compute failure links $\mathrm{fail}$ for all states in $Q$
\State Mark state $q$ as forbidden if $q$ corresponds to a complete forbidden string
\State Propagate forbidden marks so that if $\mathrm{fail}(q)$ is forbidden, then $q$ is also forbidden

\Lcomment{Initialize tables for base tokens}
\State Let $\Gamma_{\mathrm{base}}$ be the tokens in $\Gamma$ that are not produced by BPE merges
\ForAll{$q\in Q$}
    \ForAll{$w\in\Gamma_{\mathrm{base}}$}
        \State{Read $w$ from $q$ in the AC trie to obtain the resulting state $q'$ and the blocking indicator $b$.}
        \State $\Delta[q,w]\gets q'$
        \State $B[q,w]\gets b$
    \EndFor
\EndFor

\Lcomment{Compose tables for merged tokens}
\ForAll{merge rule $(u,v)\mapsto w$ in BPE order}
    \ForAll{$q\in Q$}
        \State $q'\gets \Delta[q,u]$
        \State $\Delta[q,w]\gets \Delta[q',v]$
        \State $B[q,w]\gets B[q,u]\lor B[q',v]$
    \EndFor
\EndFor

\State Mark $B[q,\texttt{EOS}]\gets 0$ for all $q\in Q$
\State \Return $\Delta,B,q_{\mathrm{root}}$
\end{algorithmic}
\end{algorithm}

\subsubsection{Regex Constraints}

\paragraph{Precomputation.}
For regex constraints, the BPE-aware precomputation must compose 
both token transitions and suffix-reachable sets. 
Let a merged token $w$ be formed by concatenating 
two previously constructed tokens $u$ and $v$, so that $w=u\cdot v$.
For a DFA $A_i$, reading $w$ from state $q$ is equivalent to
first reading $u$ from $q$ and then reading $v$ from the resulting state. 
Thus, when $\Delta_i[q,u]\neq\bot$, the transition can be computed as
\[
    \Delta_i[q,w] = \Delta_i[\Delta_i[q,u],v].
\]
If $\Delta_i[q,u]=\bot$, then the run has already entered 
the omitted non-accepting sink state, so $\Delta_i[q,w]=\bot$.
The blocking decision is composed in the same way. 
Token $w$ is blocked from state $q$ 
if reading $u$ from $q$ reaches an accepting state, 
or if reading $v$ from the intermediate state $\Delta_i[q,u]$
reaches an accepting state.
Hence, when $\Delta_i[q,u]\neq\bot$,
\[
    B_i[q,w] = B_i[q,u] \lor B_i[\Delta_i[q,u],v].
\]
When $\Delta_i[q,u]=\bot$, 
we only keep the blocking decision from reading $u$, 
since the remaining part cannot lead back from the omitted sink state.
In addition to transitions and blocking masks, 
regex constraints require suffix-reachable sets. 
The suffixes of $w=u\cdot v$ are either suffixes of $v$, 
or strings obtained by taking a suffix of $u$ and then appending $v$.
Therefore,
\[
    R_i[w]
    =
    R_i[v]
    \cup
    \{\Delta_i[q,v] : q\in R_i[u],\ \Delta_i[q,v]\neq\bot\}.
\]
This recurrence allows \ours{} to compute merged-token tables using only previously
computed tables for $u$ and $v$.

\begin{algorithm}
\caption{BPE-Aware Precomputation for Forbidden Regex Constraints}
\label{NCO:app:alg:bpe_dfa_precompute}
\begin{algorithmic}[1]
\Require Forbidden DFAs $D=\{A_1,\ldots,A_m\}$, tokenizer vocabulary $\Gamma$, BPE merge rules $\mathcal{R}$
\Ensure Token transitions $\{\Delta_i\}_{i=1}^{m}$, blocking masks $\{B_i\}_{i=1}^{m}$, and suffix-reachable sets $\{R_i\}_{i=1}^{m}$

\ForAll{$A_i=(Q_i,\Sigma,\delta_i,s_i,F_i)\in D$}
    \State Let $\Gamma_{\mathrm{base}}$ be the tokens in $\Gamma$ that are not produced by BPE merges

    \Lcomment{Initialize tables for base tokens}
    \ForAll{$q\in Q_i$}
        \ForAll{$w\in\Gamma_{\mathrm{base}}$}
            \State{Read $w$ from state $q$ in $A_i$ to obtain the resulting state $q'$ and the blocking indicator $b$.}
            \State{Set $\Delta_i[q,w]\gets q'$ and $B_i[q,w]\gets b$.}
        \EndFor
    \EndFor

    \Lcomment{Initialize suffix-reachable sets for base tokens}
    \ForAll{$w\in\Gamma_{\mathrm{base}}$}
        \State{Scan all suffixes of $w$ from $s_i$ to compute $R_i[w]$ and the suffix blocking indicator $b_{\mathrm{suf}}$.}
        \If{$b_{\mathrm{suf}}=1$}
            \State{Mark $B_i[q,w]$ as $1$ for all states $q\in Q_i$.}
        \EndIf
    \EndFor

    \Lcomment{Compose tables for merged tokens}
    \ForAll{merge rule $(u,v)\mapsto w$ in BPE order}
        \ForAll{$q\in Q_i$}
            \State{Let $q'$ be $\Delta_i[q,u]$, the state obtained by reading $u$ from $q$.}
            \State{Set $\Delta_i[q,w]$ to $\bot$ and inherit $B_i[q,w]$ from $B_i[q,u]$ if $q'=\bot$.}
            \State{Otherwise, set $\Delta_i[q,w]$ to $\Delta_i[q',v]$ and $B_i[q,w]$ to $B_i[q,u]\lor B_i[q',v]$.}
        \EndFor

        \State{Initialize $R_i[w]$ with $R_i[v]$.}
        \ForAll{$q\in R_i[u]$}
            \State{Add $\Delta_i[q,v]$ to $R_i[w]$ if $\Delta_i[q,v]\neq\bot$.}
        \EndFor

        \If{$B_i[s_i,v]=1$ or there exists $q\in R_i[u]$ such that $B_i[q,v]=1$}
            \State{Mark $B_i[q,w]$ as $1$ for all states $q\in Q_i$.}
        \EndIf
    \EndFor
    \State Mark $B_i[q,\texttt{EOS}]$ as $0$ for all $q\in Q_i.$
\EndFor
\State \Return $\{\Delta_i,B_i,R_i\}_{i=1}^{m}$
\end{algorithmic}
\end{algorithm}

\subsection{MatMul Trick}

The naive DFA decoding algorithm checks every candidate token 
against every active DFA state. 
This directly implements the validity condition, 
but it can be inefficient when the vocabulary is large 
or when multiple DFA constraints are used. 
\ours{} accelerates this step 
by representing active states and blocking masks as matrices. 
This converts the token-blocking computation into a dense matrix multiplication 
that can be executed efficiently on GPUs.

Let $D=\{A_1,\ldots,A_m\}$ be the set of forbidden DFAs. 
We concatenate the state spaces of all DFAs 
into a single global state index set
\[
    Q_{\mathrm{all}} = Q_1 \sqcup \cdots \sqcup Q_m,
\]
where states are re-indexed so that the sets are disjoint. 
At decoding step $t$, let
\[
    \mathbf{s}_t \in \{0,1\}^{|Q_{\mathrm{all}}|}
\]
be the active-state vector, 
where $\mathbf{s}_t[q]=1$ indicates that global state $q$ is active. 
We also concatenate the precomputed blocking masks into a matrix
\[
    \mathbf{B}\in\{0,1\}^{|Q_{\mathrm{all}}|\times |\Gamma|},
\]
where $\mathbf{B}[q,w]=1$ means that token $w$ is blocked 
when state $q$ is active. 
Then the blocking vector over the vocabulary is computed by
\[
    \mathbf{b}_t = \mathbf{s}_t^\top \mathbf{B}.
\]
A token $w$ is invalid if $\mathbf{b}_t[w] > 0$.
In practice, \ours{} adds a large negative value to the logits of these tokens. 
This gives the same result as iterating over all active states 
and taking the union of their blocked tokens,
but exposes the computation as a GPU-friendly matrix multiplication.

The suffix-reachable sets are also represented as a matrix
\[
    \mathbf{R}\in\{0,1\}^{|\Gamma|\times |Q_{\mathrm{all}}|}, 
\]
where $\mathbf{R}[w,q]=1$ indicates that state $q$ is 
reachable from the initial state of its DFA 
by reading a suffix of token $w$.
For batched decoding, the active state vectors for all sequences are stacked
into a matrix $\mathbf{S}_t\in\{0,1\}^{B\times |Q_{\mathrm{all}}|}$. 
The blocked token matrix is then computed as
\[
    \mathbf{H}_t = \mathbf{S}_t\mathbf{B},
\]
where $\mathbf{H}_t\in\mathbb{R}^{B\times |\Gamma|}$. 
Tokens with positive entries in $\mathbf{H}_t$ are masked 
for the corresponding batch element. 
This batched form is the one used in our implementation.

\paragraph{Decoding.}
Algorithm~\ref{NCO:app:alg:dfa_decode_matmul} gives 
the matrix-based decoding procedure for regex constraints. 
Compared to Algorithm~\ref{NCO:app:alg:dfa_decode}, 
the validity filtering step is replaced by a matrix multiplication 
between the active-state vector and the blocking mask matrix. 
The active-state update remains the same. 
For each selected token, \ours{} adds states reachable from suffixes of the token 
and advances the previous active states through the token. 
In beam search, the active-state vectors are reordered 
according to the selected beam indices in the same way as the model cache.

\begin{algorithm}
\caption{MatMul-Based Decoding with Forbidden Regex Constraints}
\label{NCO:app:alg:dfa_decode_matmul}
\begin{algorithmic}[1]
\Require Model $M_\theta$, input string $X$, tokenizer vocabulary $\Gamma$, decoding rule $\mathsf{Decode}$
\Require Concatenated transition table $\Delta$, blocking matrix $\mathbf{B}$, suffix-reachable matrix $\mathbf{R}$, initial-state vector $\mathbf{s}_{\mathrm{init}}$
\Ensure Output string $Y$ satisfying $\mathrm{Valid}_{D}$

\State $Y_0\gets\epsilon$
\State $\mathbf{s}_0\gets \mathbf{s}_{\mathrm{init}}$

\For{$t=0,1,\ldots$ until termination}
    \State $\ell_t\gets$ next-token logits from $M_\theta$ conditioned on $(X,Y_t)$

    \Lcomment{Compute blocked tokens by matrix multiplication}
    \State $\mathbf{b}_t\gets \mathbf{s}_t^\top\mathbf{B}$
    \ForAll{$w\in\Gamma$}
        \If{$\mathbf{b}_t[w]>0$}
            \State $\ell_t(w)\gets -\infty$ \Rcomment{$w\notin\Gamma_D(Y_t)$}
        \EndIf
    \EndFor

    \State $y_{t+1}\gets \mathsf{Decode}(\ell_t)$
    \State $Y_{t+1}\gets Y_t\cdot y_{t+1}$

    \Lcomment{Update active states}
    \State $\mathbf{s}_{t+1}\gets \mathbf{R}[y_{t+1}]\lor \mathbf{s}_{\mathrm{init}}$
    \ForAll{active states $q$ with $\mathbf{s}_t[q]=1$}
        \If{$\Delta[q,y_{t+1}]\neq\bot$}
            \State $\mathbf{s}_{t+1}[\Delta[q,y_{t+1}]]\gets 1$
        \EndIf
    \EndFor
\EndFor
\State \Return $Y$
\end{algorithmic}
\end{algorithm}

\subsection{Compatibility with Other Decoding Strategies}

\ours{} is designed as a logit processor 
for standard generation pipelines.
At each step, it modifies only the token scores by masking invalid candidates.
The subsequent token selection can therefore be performed 
by the original decoding strategy, including greedy decoding, 
top-$k$ sampling, top-$p$ sampling, temperature sampling, and beam search.
No model training or parameter update is required.

The same constraint structure is reused across batch elements and beams.
For batched decoding, each sequence stores its own constraint state,
while the precomputed transition and mask tables are shared across the batch.
For finite hard constraints, each sequence or beam stores 
only its current Aho-Corasick trie state. 
For regex constraints, each sequence or beam stores
its active DFA state sets. 
When a beam is expanded, the corresponding constraint state is copied together 
with the generated prefix and updated after a new token is selected. 
This makes \ours{} compatible with batched decoding and beam search
without rebuilding the constraint structures at every step.

In beam search, local logit masking can distort normalized probabilities
by removing probability mass from invalid tokens. 
\ours{} reduces this effect by using the mask only as a feasibility filter. 
Beam scores are accumulated
using the original model log probabilities of the selected valid tokens.
This preserves constraint satisfaction 
while reducing unnecessary distortion in comparisons across beams.

\section{Proof and Runtime Analysis}
\label{NCO:app:sec:proof_and_runtime_analysis}
\subsection{Correctness}
\subsubsection{Multiple Finite Hard Constraints}

\thmStringCorrectness*

\begin{proof}
Let $Y_t$ be the output prefix after $t$ decoding steps, 
and let $q_t$ be the state of the Aho-Corasick automaton after reading $Y_t$.
We use the standard invariant of the Aho-Corasick automaton.
The state $q_t$ represents the longest suffix of $Y_t$ 
that is also a prefix of some forbidden string in $P$.
Moreover, while reading an additional string $w$ from $q_t$, 
the automaton reaches an output state 
if and only if 
some forbidden string in $P$ is completed as a substring of $Y_t w$.

By construction, \ours{} masks a token $w$ 
whenever reading the string represented by $w$ from $q_t$ 
reaches an output state.
Therefore, every unmasked token $w$ satisfies the following property.
If $Y_t$ contains no forbidden string, 
then $Y_t w$ also contains no forbidden string.

We prove the theorem by induction on $t$.
The empty prefix $Y_0$ contains no forbidden string.
Assume that $Y_t$ contains no forbidden string.
The next token $w_t$ is sampled only from the unmasked tokens.
Hence, reading $w_t$ from $q_t$ does not reach an output state.
By the invariant above, no forbidden string is completed while reading $w_t$.
Since $Y_t$ was valid by the induction hypothesis, 
the updated prefix $Y_{t+1}=Y_t w_t$ also contains no forbidden string.
Thus every generated prefix is valid, 
and the final output contains no string in $P$ as a substring.
\end{proof}

\subsubsection{Multiple Regex Constraints}

\thmRegexCorrectness*

\begin{proof}
Fix an automaton $A_i$.
For an output prefix $Y_t$, 
let $S_t^i\subseteq Q_i$ be the active state set maintained by \ours{}.
We prove the invariant
\[
    S_t^i
    =
    \{\delta_i(s_i,z) \mid z \text{ is a suffix of } Y_t\}.
\]
That is, $S_t^i$ contains exactly the states reached by running $A_i$ 
on every suffix of the current output prefix.

The invariant holds for $Y_0=\epsilon$, 
since the only suffix of the empty string is $\epsilon$, 
and the corresponding state is $s_i$.
Suppose the invariant holds for $Y_t$.
After appending a token $w$, 
every suffix of $Y_t w$ is of one of two forms.
It either extends a suffix of $Y_t$ by $w$, 
or it starts inside the newly appended token $w$.
The first case is captured 
by applying the token transition to every state in $S_t^i$.
The second case is captured 
by the precomputed suffix transitions of $w$ from the start state $s_i$.
Therefore, the update rule of \ours{} preserves the invariant for
$S_{t+1}^i$.

Now consider the token mask.
For each automaton $A_i$, 
\ours{} masks a token $w$ if reading $w$ 
from some state in $S_t^i$, or reading a suffix of $w$ from $s_i$, 
reaches an accepting state in $F_i$.
By the invariant, this is exactly the condition 
that some substring of $Y_t w$ accepted by $A_i$ is completed 
while reading $w$.
The final mask is obtained by combining the masks of all automata, 
so a token is allowed only when it is valid for every $A_i$.

We prove correctness by induction on $t$.
The empty prefix contains no substring accepted by any automaton in $D$.
Assume that $Y_t$ satisfies all regex constraints.
Let $w_t$ be the next token selected by \ours{}.
Since $w_t$ is unmasked, for every $A_i$, 
appending $w_t$ does not complete an accepted substring. 
Indeed, no run starting from an active state in $S_{i,t}$ 
reaches an accepting state while reading $w_t$, 
and no run starting inside $w_t$ is accepted.
Thus no substring accepted by $A_i$ is completed in $Y_t w_t$.
Since no such substring existed in $Y_t$ by the induction hypothesis, the new
prefix $Y_{t+1}=Y_t w_t$ also satisfies all regex constraints.
Therefore, every generated prefix and the final output contain no substring
accepted by any forbidden DFA.
\end{proof}

\subsection{Runtime Analysis}

\subsubsection{Multiple Finite Hard Constraints}
\paragraph{Precomputation for finite hard constraints.}
Let $P$ be a finite set of forbidden strings, 
let $L$ denote the total length of all strings in $P$, that is
\[
    L=\sum_{w\in P}|w|,
\] 
and let $|\Gamma|$ be the vocabulary size.
For multiple hard constraints, 
\ours{} constructs the Aho-Corasick trie in $O(L)$ time 
and performs BPE-based token-level precomputation in $O(|\Gamma|L)$ time.

The Aho-Corasick trie contains one node 
for each distinct prefix of a forbidden string.
Thus, the number of trie states is $O(L)$.
The trie and its failure links can therefore be constructed in $O(L)$ time 
over a fixed alphabet.
After the trie is constructed, \ours{} precomputes token-level transitions.
For each trie state and each vocabulary token, 
the table stores the state reached after reading that token and 
whether a forbidden state is encountered during the transition.
Without using the BPE structure, computing this table 
by scanning each token from every trie state 
would introduce an additional factor depending on the maximum token length.
With the BPE-based dynamic programming procedure, 
the transition of a merged token is computed 
by composing the transitions of its two children.
Therefore, each vocabulary token is processed once per trie state, 
giving $O(|\Gamma|L)$ token-level precomputation time.

\paragraph{Inference for finite hard constraints.}
Let $N$ be the number of generated tokens.
For multiple forbidden strings, after preprocessing, 
\ours{} updates the constraint state in $O(1)$ time per generated token 
and $O(N)$ time over the full generated sequence, 
excluding the dense operation of applying the mask to the vocabulary logits.

During inference, \ours{} maintains a single Aho-Corasick state 
for each generated sequence.
Once a token is selected, the next constraint state is obtained 
by one lookup in the precomputed token transition table.
Therefore, the state update cost is $O(1)$ per generated token 
and $O(N)$ over $N$ generated tokens.
The invalid token set for the current state is also available 
from the precomputed mask table.
Thus, generating the constraint mask does not require 
scanning the generated text or traversing the trie at inference time.
If the dense operation of applying the mask to all vocabulary logits is counted,
it adds the usual $O(|\Gamma|)$ vector operation per decoding step.

\subsubsection{Multiple Regex Constraints}
\paragraph{Precomputation for regex constraints.}
Let $D=\{A_1,\ldots,A_l\}$ be a finite set of forbidden DFAs, 
where $A_i=(Q_i,\Sigma,\delta_i,s_i,F_i)$, and let $M$ denote the total number of states across all automata in $D$, that is
\[
    M=\sum_{i=1}^{l}|Q_i|.
\]
For multiple regex constraints, 
\ours{} performs token-level transition precomputation 
in $O(M|\Gamma|)$ time using the BPE-based composition, 
and stores the resulting tables additively across automata.

For each DFA $A_i$, \ours{} constructs a token-level transition table 
that stores the result of reading each vocabulary token 
from each state of $A_i$.
Using the BPE-based composition, 
the transition of a merged token is computed
from the transitions of its two children.
Thus, for automaton $A_i$, the token-level transition table is computed 
in $O(|Q_i||\Gamma|)$ time.
Summing this cost over all automata gives
\[
    \sum_{i=1}^{l} O(|Q_i||\Gamma|)
    =
    O(M|\Gamma|).
\]
The tables are stored separately or concatenated across automata,
so the storage and preprocessing cost are additive 
in the number of automaton states.
This avoids the product state space 
that would arise from explicitly intersecting the automata.

\paragraph{Inference for regex constraints.}
Let $D=\{A_1,\ldots,A_l\}$ be a finite set of forbidden DFAs 
and let $M=\sum_{i=1}^{l}|Q_i|$.
After preprocessing, \ours{} updates the regex constraint state 
in $O(M)$ time per generated token, 
excluding GPU parallel mask composition, 
and avoids constructing the explicit product automaton 
whose state space may have size
\[
    \prod_{i=1}^{l}|Q_i|.
\]

During inference, \ours{} maintains an active state set for each DFA.
For automaton $A_i$, the update after a generated token has two components.
First, previous active states are propagated 
by the precomputed token transition.
Second, states reachable from the start state 
by suffixes of the generated token are added, 
since a forbidden substring may start inside the newly generated token.
Both components are obtained from precomputed tables.
Under the active state representation, 
updating the state information for $A_i$ is linear in $|Q_i|$.
Summing over all automata gives
\[
    \sum_{i=1}^{l} O(|Q_i|)=O(M)
\]
time per generated token.
The token masks are computed from the active state representation 
and the precomputed forbidden mask tables.
When written as a dense operation, 
this mask composition has arithmetic size proportional to $M|\Gamma|$, 
but \ours{} performs it as a GPU parallel operation.
The sequential state update therefore scales additively 
with the total number of states.
In contrast, a direct construction for simultaneous enforcement 
would intersect the constraint automata.
The resulting product state space has size
\[
    |Q_1|\cdot |Q_2|\cdots |Q_l|.
\]
Therefore, \ours{} avoids the explicit product automaton 
and maintains only the separate active state representations.

\section{Experimental Details}
\label{NCO:app:sec:experimental_details}

We provide the experimental setup used throughout the paper, 
including datasets, constraint sets, models, metrics, baseline implementations, 
ablation variants, decoding hyperparameters, and the measurement environment. 
All methods are evaluated with the same prompts, models, 
and decoding settings unless stated otherwise. 
Additional results in Appendix~\ref{NCO:app:sec:additional_experimental_results} follow 
the same setup.

\subsection{Dataset and Constraint Details}
\label{NCO:app:ssec:dataset_constraint_details}

\paragraph{RTP profanity suppression.}
For the finite hard constraint task, we use
RealToxicityPrompts~\citep{GehmanGSCS2020}~(RTP). 
We use the English subset of the 
LDNOOBW~\citep{LDNOOBW}
bad-word lexicon as the forbidden string set. 
The lexicon is used in full without additional filtering or manual modification. 
RTP provides prompts paired with reference continuations. 
We select prompts whose reference continuation contains 
at least one substring matched by the LDNOOBW lexicon, 
and sample 500 prompts for evaluation.
This filtering step creates an evaluation set 
where the prompt context is likely to elicit profanity-like continuations 
from an unconstrained model.
Each selected prompt is formatted with the native chat template 
provided by the corresponding model tokenizer before generation. 
The model then generates a continuation, 
and the generated response is considered a violation 
if it contains any lexicon entry as a substring.

This task evaluates 
whether a decoding method can prevent explicit forbidden strings 
from appearing anywhere in the generated output. 
The setting is not equivalent to static token blocking, 
since a forbidden word may be split across multiple BPE tokens. 
A valid decoding method must therefore track partial matches 
across token boundaries.

\paragraph{Enron PII suppression.}
For the forbidden regex constraint task, 
we use the Enron Email Dataset~\citep{KlimtY2004}. 
We consider four common types of PII patterns, 
including email addresses, phone numbers, social security numbers, 
and credit card numbers. 
Each PII type is specified as a regular expression 
and compiled into a DFA before decoding.
The exact regex patterns are listed in Table~\ref{NCO:app:tab:pii_regex_constraints}.
Since the models used for this task are trained on the Pile, 
which includes the Enron corpus, 
prompts extracted from Enron emails  are likely 
to elicit continuations containing PII-like patterns. 
We first scan the dataset with our PII regex constraints. 
For each matched email, 
we take the text preceding the first matched span as the prompt. 
We keep prompts whose input length is between 20 and 2000 characters, 
and sample 500 prompts for evaluation. 
Given an Enron prompt, the model generates a continuation, 
and the generated response is considered a violation 
if it contains any substring that matches one of the PII regex constraints.

This task evaluates 
whether a decoding method can enforce negative regex constraints 
rather than a finite list of forbidden strings. 
The regex constraints represent structured pattern families 
with many possible instances, 
which makes enumerating all forbidden strings impractical.

\begin{table}[ht]
\centering
\caption{PII regex constraints used in the Enron experiment.}
\label{NCO:app:tab:pii_regex_constraints}
\begin{tabular}{ll}
\toprule
PII type & Regex pattern \\
\midrule
Email address & \verb|[a-zA-Z0-9._%+-]+@[a-zA-Z0-9.-]+\.[a-zA-Z]{2,}| \\
Social security number & \verb|[0-9]{3}-[0-9]{2}-[0-9]{4}| \\
Phone number & \verb|[0-9]{3}-[0-9]{3}-[0-9]{4}| \\
Credit card number & \verb|[0-9]{4}-[0-9]{4}-[0-9]{4}-[0-9]{4}| \\
\bottomrule
\end{tabular}
\end{table}

\subsection{Model Details}

We evaluate models from several families to test whether \ours{} remains
effective across different tokenizers and model architectures. 
For RTP profanity suppression, we use six instruction or chat-style models:
Llama~2~7B, Llama~3.1~8B, Qwen2.5~7B, Falcon~7B, Falcon3~7B, and Phi4~14B.
For Enron PII suppression, we use EleutherAI models trained 
on the Pile~\citep{GaoBBGHFPHTNPL2020},
including GPT-J~6B, GPT-Neo~2.7B, and Pythia~6.9B. 
The Pile includes the Enron Email Dataset, 
which makes these models likely to produce PII-like continuations 
under prompts extracted from Enron emails.
This property is useful for evaluating 
whether decoding-time constraints can remove structured PII patterns 
when the unconstrained model assigns them non-trivial probability.
We include GPT-family models for this task because prior work has commonly used
them in PII suppression settings, and because unconstrained generations from
these models produce PII-like patterns with non-trivial frequency.

All models are used only at inference time. 
We do not fine-tune, retrain, or modify model parameters for any method. 
For chat-style models, prompts are formatted with the native chat template
provided by the corresponding model tokenizer before generation.

\subsection{Decoding Hyperparameters}
Unless otherwise stated, all hard-constraint experiments use greedy decoding
with a maximum generation length of 256 tokens.
We vary the batch size for the main throughput experiments 
and use batch size $1$ for ablation experiments.
We do not use temperature sampling, top-$p$ sampling, top-$k$ sampling, 
or repetition penalty in these experiments. 
The same decoding hyperparameters are used for the base model, 
rejection sampling, GUARD, and \ours{}.

The soft logit experiments use sampling-based decoding 
instead of greedy decoding. 
We use the default sampling hyperparameters specified 
by the corresponding model generation configuration 
for temperature, top-$p$, and top-$k$. 
No repetition penalty is applied. 
The soft logit penalty is denoted by $\lambda > 0$. 
At each decoding step, tokens that would complete a forbidden string or regex
match receive an additive logit shift of $-\lambda$. 
The hard constraint setting corresponds to $\lambda=\infty$.

\subsection{Baseline Implementation Details}
\label{NCO:app:ssec:baseline_implementation_details}

We compare \ours{} against the unconstrained base model, 
rejection sampling, and GUARD. 
The base model uses the same prompts and decoding hyperparameters 
without applying any negative constraint.
For rejection sampling, we maintain the matching state 
for the current partial generation 
and check whether appending a sampled token would create a forbidden
string or a substring matching one of the regex constraints. 
If the sampled token causes a violation, we reject it, 
restore the matching state to the state before token selection, 
mask the rejected token, and resample within the same decoding step. 
This implementation avoids an additional forward inference pass 
for each rejected token, 
since resampling is performed from the already computed logits.

For GUARD, 
we reimplement the trie-based token suppression part 
that corresponds to finite hard constraints,
since we could not find an official public implementation.
Our implementation represents the forbidden string set as a trie 
and suppresses candidate tokens 
whose continuation would complete a forbidden string.
We do not reproduce GUARD's prompt classifier 
or SBERT-based semantic soft matching module, 
since our evaluation focuses on explicit negative substring constraints 
rather than semantic unlearning targets.
For the RTP task, 
GUARD is evaluated with the same forbidden substring constraints 
used by rejection sampling and \ours{}.
For the Enron PII task, 
GUARD cannot directly enforce the original regex constraints 
used by rejection sampling and \ours{}.
The regex constraints in Table~\ref{NCO:app:tab:pii_regex_constraints} 
define infinite sets of forbidden strings, 
whereas GUARD requires a finite forbidden lexicon to construct its trie.
We therefore build a finite-string approximation for GUARD.
Specifically, we first apply the same regex constraints 
used by rejection sampling and \ours{} 
to the $500$ evaluation instances sampled from Enron.
We then collect all matched PII spans into a finite lexicon 
and construct GUARD's trie from this lexicon.
This approximation allows GUARD to be evaluated on the PII task, 
but it is not equivalent to enforcing the original regex constraints.
This distinction reflects a limitation of trie-based finite-string suppression 
when the target constraint denotes an infinite set of strings 
described by a regex.

\subsection{Ablation Implementation Details}
\label{NCO:app:ssec:ablation_implementation_details}

We implement four ablation variants in total, 
two for finite hard constraints and two for regex constraints. 
Each variant removes one component from \ours{} 
while keeping the same constraint set, prompt, model, 
and decoding hyperparameters.

For finite hard constraints, 
the first ablation removes Aho-Corasick failure transitions. 
Without failure transitions, 
the decoder can no longer maintain a single deterministic matching state 
that summarizes the longest relevant suffix of the generated prefix. 
A direct implementation would need to track all possible match start positions, 
similar to the active-state representation used for regex constraints, 
which is expensive for large forbidden lexicons. 
We therefore implement this ablation with a reversed trie. 
At each decoding step, 
the method inspects the suffix of the current generated text in reverse order 
and checks whether appending a candidate token would complete any forbidden string. 
This variant preserves the same string-constraint semantics, 
but it removes the constant-time state update 
provided by Aho-Corasick failure transitions. 
The comparison therefore isolates the decoding-time benefit of failure links 
against a practical trie-based alternative.
The second string-constraint ablation removes 
BPE-based transition precomputation. 
Instead of computing token-level transitions 
by reusing the merge structure of the tokenizer, 
this variant computes the transition and mask information 
for each vocabulary token independently. 
This isolates the effect of BPE-aware precomputation on setup time.

For regex constraints, the first ablation removes the BPE-based speedup 
for suffix-reachable state precomputation. 
This table stores the DFA states reachable from the start state 
by suffixes of each vocabulary token, 
which is needed to detect matches that begin inside a newly appended token. 
Without the BPE-based speedup, 
these suffix-reachable states are computed independently for each token, 
increasing the one-time preprocessing cost.
The second regex-constraint ablation removes 
mask aggregation based on matrix multiplication. 
This variant computes the token mask 
by iterating over active DFA states and
combining their token-level validity masks sequentially, 
rather than using parallel matrix operations on the GPU. 
This isolates the contribution of GPU-parallel mask aggregation 
to decoding throughput.

\subsection{Evaluation Metrics}

We evaluate generation quality, constraint effectiveness, and
practical efficiency, using the following metrics.
For hard-constraint experiments, decoding is deterministic 
because we use greedy decoding under a fixed model, prompt set, and constraint set. 
As a result, perplexity and violation rate do not vary across repeated runs. 
We therefore report standard deviations only for time-related measurements.

\paragraph{Perplexity.}
We report perplexity as a proxy for generation quality. 
All perplexity scores are computed after generation 
using Llama~3.1~8B as a standard evaluation model, 
so that outputs from different decoding methods are compared 
under a common likelihood model. 
Given a generated sequence of length $T$ with token probabilities
$p_\theta(y_t \mid y_{<t}, x)$ under the evaluation model, 
perplexity is computed as
\begin{equation}
\mathrm{PPL} = \exp\left(-\frac{1}{T}\sum_{t=1}^{T}\log p_\theta(y_t \mid y_{<t}, x)\right).
\end{equation}
Lower perplexity indicates that the generated tokens are, on average, more
probable under the evaluation model.

\paragraph{Violation Rate.}
We measure constraint adherence using violation rate, 
defined as the fraction of generated responses 
that contain at least one forbidden pattern. 
Let $N$ be the total number of generated responses 
and let $N_{\text{v}}$ be the number of responses that violate 
at least one constraint. 
The violation rate is computed as
\begin{equation}
\mathrm{Violation\ Rate} = \frac{N_{\text{v}}}{N} \times 100.
\end{equation}
Lower violation rate indicates better compliance with the negative constraints.
Violation is determined by exact pattern matching on the generated text. 
For finite hard constraints, we check 
whether each forbidden string appears in the output 
using Python string containment. 
For regular-expression constraints, 
we use the \texttt{re} library 
to test whether the output matches any forbidden pattern.

\paragraph{Throughput.}
We measure decoding efficiency using throughput, 
defined as the number of generated tokens per second. 
For a run that generates $K$ total tokens 
and takes wall-clock time $\Delta t$ seconds
during the generation phase, 
including constraint checking and masking, 
throughput is computed as
\begin{equation}
\mathrm{Throughput} = \frac{K}{\Delta t}~(\text{tokens/s}).
\end{equation}
We refer to this value as absolute throughput.
For each method $m$ and batch size $b$, 
we also compute relative throughput as
\begin{equation}
\mathrm{Relative\ Throughput}
=\frac{T_{m,b}}{T_{\mathrm{Base},b}}\times 100,
\end{equation}
where $T_{m,b}$ is the absolute throughput of method $m$ 
at batch size $b$, 
and $T_{\mathrm{Base},b}$ is the absolute throughput 
of the base unconstrained model at the same batch size.
The Base model therefore has relative throughput $100.0$ 
by definition.
We report relative throughput in the main paper 
to compare runtime overhead across models and batch sizes.
For each setting, we repeat the throughput measurement three times 
and report the mean and standard deviation.

\subsection{Experimental Environment}
We ran all experiments on machines 
equipped with AMD Ryzen Threadripper~9960X CPUs 
and NVIDIA RTX PRO~6000 Blackwell Max-Q GPUs. 
The experiments were conducted on Rocky Linux~9.6 using Python~3.10.19, 
PyTorch~2.9.0, Transformers~4.57.3, scikit-learn~1.7.2, Datasets~4.4.1, 
Interegular~0.3.3, and Accelerate~1.11.0.

\subsection{Licenses for Existing Assets}
\label{NCO:app:ssec:license}

Table~\ref{NCO:app:tab:asset_licenses} summarizes the existing assets used in
our experiments. We use these assets only for offline evaluation and cite their
original sources in the main paper or appendix. We do not release any new dataset
or model checkpoint as part of this work.

\begin{table}[ht]
\centering
\small
\caption{Licenses and terms of use for existing assets. 
    We list all datasets, pretrained models, and software dependencies 
    used in our experiments, 
    and report the corresponding license or usage terms 
    for each asset when publicly available.
}
\label{NCO:app:tab:asset_licenses}
\begin{tabularx}{\linewidth}{lY}
\toprule[1.2pt]
\multicolumn{2}{l}{\textbf{Dataset}} \\
\midrule
Asset & License / terms \\
\midrule
RealToxicityPrompts~\citep{GehmanGSCS2020} & Apache-2.0. \\
Enron Email Dataset~\citep{KlimtY2004} & Publicly released but license not specified. \\
Alpaca~\citep{TaoriGZDLGLH2023} & CC BY-NC 4.0 \\
RegexPSPACE~\citep{JinHH2025} & Publicly released but license not specified. \\
LDNOOBW~\citep{LDNOOBW} & CC BY 4.0. \\
\midrule[1.2pt]
\multicolumn{2}{l}{\textbf{Model}} \\
\midrule
Asset & License / terms \\
\midrule
Llama~2~7B~\citep{Touvron2023Llama2} & Llama 2 Community License. \\
Llama~3.1~8B~\citep{Aaron2024Llama3} & Llama 3.1 Community License. \\
Qwen2.5~7B~\citep{Qwen2024Qwen25} & Apache-2.0. \\
Falcon~7B~\citep{Almazrouei2023Falcon} & Apache-2.0. \\
Falcon3~7B~\citep{TII2024Falcon3} & TII Falcon-LLM License 2.0. \\
Phi4~14B~\citep{Microsoft2024Phi4} & MIT License. \\
GPT-J~6B~\citep{WangKomatsuzaki2021GPTJ} & Apache-2.0. \\
GPT-Neo~2.7B~\citep{Black2021GPTNeo} & MIT License. \\
Pythia~6.9B~\citep{Biderman2023Pythia} & Apache-2.0. \\
\midrule[1.2pt]
\multicolumn{2}{l}{\textbf{Software}} \\
\midrule
Asset & License / terms \\
\midrule
Python~3.10.19 & PSF License Version 2. \\
PyTorch~2.9.0 & BSD-3-Clause. \\
Transformers~4.57.3 & Apache-2.0. \\
scikit-learn~1.7.2 & BSD-3-Clause. \\
Datasets~4.4.1 & Apache-2.0. \\
Interegular~0.3.3 & MIT License. \\
Accelerate~1.11.0 & Apache-2.0. \\
\bottomrule[1.2pt]
\end{tabularx}
\end{table}

\section{Additional Experimental Results}
\label{NCO:app:sec:additional_experimental_results}

\subsection{Absolute Throughput Results on Main Tasks}
\label{NCO:app:ssec:absolute_throughput_main_tasks}

We report absolute throughput results 
for the main profanity suppression and PII suppression tasks.
These results complement the relative throughput results 
reported in the main paper.
The same two datasets and nine models are used.
For each setting, we report the throughput of the Base model, 
rejection sampling, GUARD, and \ours{}.
For the RTP profanity suppression task, 
we evaluate batch sizes from 1 to 64.
For the Enron PII suppression task, 
we evaluate batch sizes from 1 to 16.
Each batch size is chosen from powers of two.
For each model and task, 
we increase the batch size up to the largest value 
that fits in memory.
Enron uses smaller maximum batch sizes than RTP 
because its prompts are longer and require more memory during generation.

Tables~\ref{NCO:app:tab:profanity_main_results_absolute}
and~\ref{NCO:app:tab:pii_main_results_absolute} 
show the absolute throughput results.
The results show that \ours{} preserves most of the throughput 
of the base unconstrained model across batch sizes.
Rejection sampling incurs larger overhead 
when invalid continuations are sampled more frequently.
GUARD also shows noticeable throughput degradation under batched decoding 
because its trie-based restriction procedure does not scale 
as efficiently as the precomputed masking strategy used by \ours{}.
These results support the relative throughput trends 
reported in the main paper.

\begin{table*}[ht]
    \centering
    \small
    \caption{
        Absolute throughput on the RTP profanity suppression task.
        The table reports Base, rejection sampling, GUARD, and \ours{}.
        Batch sizes range from 1 to 64.
        RS denotes rejection sampling.
    }
    \setlength{\tabcolsep}{4.3pt}
    \label{NCO:app:tab:profanity_main_results_absolute}
    \begin{tabular}{ccrrrrrrr}
        \toprule[1.2pt]
        \multirow{2}{*}[-0.5ex]{Model} & \multirow{2}{*}[-0.5ex]{Method} & \multicolumn{7}{c}{Batch Size} \\
        \cmidrule{3-9}
        \multicolumn{2}{c}{} & \multicolumn{1}{c}{1} & \multicolumn{1}{c}{2} & \multicolumn{1}{c}{4} & \multicolumn{1}{c}{8} & \multicolumn{1}{c}{16} & \multicolumn{1}{c}{32} & \multicolumn{1}{c}{64} \\
        \midrule
        \multirow{4}{*}{\makecell[c]{Llama~2~7B\\(14.8\%)}}
        & Base & 50.3\tiny{$\pm$0.0} & 89.5\tiny{$\pm$0.1} & 176.4\tiny{$\pm$0.5} & 309.1\tiny{$\pm$0.5} & 516.2\tiny{$\pm$2.8} & 751.7\tiny{$\pm$5.1} & 977.6\tiny{$\pm$7.6} \\
        & RS & 50.0\tiny{$\pm$0.0} & 88.4\tiny{$\pm$0.3} & 173.1\tiny{$\pm$0.6} & 301.5\tiny{$\pm$0.4} & 499.1\tiny{$\pm$1.9} & 711.9\tiny{$\pm$3.5} & 918.0\tiny{$\pm$5.7} \\
        & GUARD & 35.3\tiny{$\pm$0.1} & 49.6\tiny{$\pm$0.8} & 64.7\tiny{$\pm$1.8} & 76.5\tiny{$\pm$1.4} & 84.2\tiny{$\pm$2.5} & 98.0\tiny{$\pm$1.5} & 100.0\tiny{$\pm$1.6} \\
        & \CCG\ours{} & \CCG50.3\tiny{$\pm$0.0} & \CCG89.1\tiny{$\pm$0.5} & \CCG176.3\tiny{$\pm$0.2} & \CCG309.7\tiny{$\pm$1.0} & \CCG517.4\tiny{$\pm$1.5} & \CCG750.9\tiny{$\pm$2.3} & \CCG980.1\tiny{$\pm$8.6} \\
        \midrule
        \multirow{4}{*}{\makecell[c]{Llama~3.1~8B\\(15.0\%)}}
        & Base & 45.0\tiny{$\pm$0.0} & 81.9\tiny{$\pm$0.4} & 159.3\tiny{$\pm$0.5} & 263.9\tiny{$\pm$0.8} & 457.2\tiny{$\pm$1.4} & 676.4\tiny{$\pm$5.6} & 908.8\tiny{$\pm$7.3} \\
        & RS & 44.7\tiny{$\pm$0.0} & 80.8\tiny{$\pm$0.1} & 156.9\tiny{$\pm$0.4} & 259.1\tiny{$\pm$0.2} & 444.7\tiny{$\pm$0.5} & 651.1\tiny{$\pm$5.3} & 862.9\tiny{$\pm$7.1} \\
        & GUARD & 17.3\tiny{$\pm$0.2} & 19.6\tiny{$\pm$0.4} & 22.0\tiny{$\pm$0.5} & 23.0\tiny{$\pm$0.6} & 24.0\tiny{$\pm$1.2} & 27.2\tiny{$\pm$0.5} & 27.0\tiny{$\pm$0.4} \\
        & \CCG\ours{} & \CCG44.9\tiny{$\pm$0.0} & \CCG81.6\tiny{$\pm$0.2} & \CCG158.9\tiny{$\pm$0.0} & \CCG264.1\tiny{$\pm$0.1} & \CCG458.8\tiny{$\pm$1.5} & \CCG678.2\tiny{$\pm$5.7} & \CCG911.6\tiny{$\pm$4.5} \\
        \midrule
        \multirow{4}{*}{\makecell[c]{Qwen~2.5~7B\\(17.0\%)}}
        & Base & 46.1\tiny{$\pm$0.6} & 85.5\tiny{$\pm$0.6} & 166.3\tiny{$\pm$1.7} & 311.6\tiny{$\pm$4.1} & 569.9\tiny{$\pm$5.2} & 946.6\tiny{$\pm$5.7} & 1401.5\tiny{$\pm$39.4} \\
        & RS & 46.3\tiny{$\pm$0.0} & 83.5\tiny{$\pm$0.4} & 163.8\tiny{$\pm$0.7} & 302.2\tiny{$\pm$4.8} & 544.1\tiny{$\pm$5.9} & 885.5\tiny{$\pm$3.6} & 1282.3\tiny{$\pm$41.0} \\
        & GUARD & 16.4\tiny{$\pm$0.0} & 16.3\tiny{$\pm$0.1} & 18.3\tiny{$\pm$0.3} & 20.2\tiny{$\pm$0.9} & 20.6\tiny{$\pm$1.2} & 20.4\tiny{$\pm$0.2} & 20.7\tiny{$\pm$2.6} \\
        & \CCG\ours{} & \CCG46.4\tiny{$\pm$0.1} & \CCG85.7\tiny{$\pm$0.4} & \CCG168.7\tiny{$\pm$1.7} & \CCG314.2\tiny{$\pm$2.8} & \CCG573.2\tiny{$\pm$2.5} & \CCG959.1\tiny{$\pm$8.3} & \CCG1411.9\tiny{$\pm$29.5} \\
        \midrule
        \multirow{4}{*}{\makecell[c]{Falcon~7B\\(11.6\%)}}
        & Base & 47.7\tiny{$\pm$0.7} & 91.3\tiny{$\pm$0.1} & 176.2\tiny{$\pm$1.6} & 320.7\tiny{$\pm$2.5} & 593.4\tiny{$\pm$2.6} & 971.4\tiny{$\pm$4.3} & 1459.3\tiny{$\pm$7.0} \\
        & RS & 48.1\tiny{$\pm$0.1} & 89.8\tiny{$\pm$0.5} & 170.0\tiny{$\pm$1.0} & 313.2\tiny{$\pm$2.1}& 569.8\tiny{$\pm$2.7} & 910.4\tiny{$\pm$11.0} & 1317.3\tiny{$\pm$57.9} \\
        & GUARD & 25.5\tiny{$\pm$0.3} & 32.1\tiny{$\pm$0.3} & 35.4\tiny{$\pm$0.8} & 41.1\tiny{$\pm$0.4} & 44.2\tiny{$\pm$1.9} & 47.2\tiny{$\pm$0.9} & 48.0\tiny{$\pm$0.5} \\
        & \CCG\ours{} & \CCG48.2\tiny{$\pm$0.1} & \CCG91.0\tiny{$\pm$0.5} & \CCG177.3\tiny{$\pm$0.7} & \CCG323.0\tiny{$\pm$2.4} & \CCG596.6\tiny{$\pm$3.8} & \CCG971.5\tiny{$\pm$14.8} & \CCG1440.0\tiny{$\pm$52.5} \\
        \midrule
        \multirow{4}{*}{\makecell[c]{Falcon~3~7B\\(14.6\%)}} 
        & Base & 45.8\tiny{$\pm$0.6} & 87.2\tiny{$\pm$1.0} & 170.7\tiny{$\pm$2.0} & 316.6\tiny{$\pm$5.0} & 597.0\tiny{$\pm$10.5} & 996.7\tiny{$\pm$27.1} & 1444.8\tiny{$\pm$70.8} \\
        & RS & 46.1\tiny{$\pm$0.1} & 85.7\tiny{$\pm$0.6} & 167.4\tiny{$\pm$0.7} & 307.2\tiny{$\pm$3.5} & 573.5\tiny{$\pm$2.8} & 946.4\tiny{$\pm$3.4} & 1340.5\tiny{$\pm$28.7} \\
        & GUARD & 16.7\tiny{$\pm$0.1} & 18.2\tiny{$\pm$0.2} & 20.8\tiny{$\pm$0.6} & 22.2\tiny{$\pm$0.9} & 23.7\tiny{$\pm$1.1} & 24.9\tiny{$\pm$1.3} & 26.0\tiny{$\pm$0.8} \\
        & \CCG\ours{} & \CCG46.2\tiny{$\pm$0.0} & \CCG87.6\tiny{$\pm$0.5} & \CCG171.7\tiny{$\pm$1.5} & \CCG321.1\tiny{$\pm$2.2} & \CCG607.9\tiny{$\pm$3.2} & \CCG1016.4\tiny{$\pm$6.2} & \CCG1494.4\tiny{$\pm$15.0} \\
        \midrule
        \multirow{4}{*}{\makecell[c]{Phi-4~14B\\(16.6\%)}} 
        & Base & 23.8\tiny{$\pm$0.2} & 48.0\tiny{$\pm$0.6} & 90.9\tiny{$\pm$0.8} & 168.8\tiny{$\pm$5.5} & 298.3\tiny{$\pm$6.4} & 474.8\tiny{$\pm$4.4} & 686.4\tiny{$\pm$22.6} \\
        & RS & 23.9\tiny{$\pm$0.1} & 47.3\tiny{$\pm$0.2} & 90.3\tiny{$\pm$0.4} & 167.2\tiny{$\pm$3.7} & 295.6\tiny{$\pm$7.4} & 462.3\tiny{$\pm$11.0} & 660.7\tiny{$\pm$17.6} \\
        & GUARD & 14.5\tiny{$\pm$0.2} & 18.2\tiny{$\pm$0.2} & 22.2\tiny{$\pm$0.4} & 25.4\tiny{$\pm$1.3} & 27.3\tiny{$\pm$1.6} & 27.8\tiny{$\pm$2.2} & 29.2\tiny{$\pm$1.8} \\
        & \CCG\ours{} & \CCG23.8\tiny{$\pm$0.0} & \CCG47.7\tiny{$\pm$0.2} & \CCG91.3\tiny{$\pm$2.4} & \CCG171.1\tiny{$\pm$3.1} & \CCG303.3\tiny{$\pm$7.9} & \CCG481.9\tiny{$\pm$5.3} & \CCG685.8\tiny{$\pm$21.4} \\
        \bottomrule[1.2pt]
    \end{tabular}
\end{table*}

\begin{table*}[ht]
    \centering
    \caption{
    Absolute throughput on the Enron PII suppression task.
    The table reports Base, rejection sampling, GUARD, and \ours{}.
    Batch sizes range from 1 to 16.
    GUARD uses finite hard constraints extracted from matched PII spans.
    }
    \label{NCO:app:tab:pii_main_results_absolute}
    \begin{tabular}{ccrrrrr}
        \toprule[1.2pt]
        \multirow{2}{*}[-0.5ex]{Model} & \multirow{2}{*}[-0.5ex]{Method} & \multicolumn{5}{c}{Batch Size} \\
        \cmidrule{3-7}
        \multicolumn{2}{c}{} & \multicolumn{1}{c}{1} & \multicolumn{1}{c}{2} & \multicolumn{1}{c}{4} & \multicolumn{1}{c}{8} & \multicolumn{1}{c}{16} \\
        \midrule
        \multirow{4}{*}{\makecell[c]{GPT-J~6B\\(87.0\%)}}
        & Base & 54.2\footnotesize{$\pm$0.0} & 88.7\footnotesize{$\pm$0.1} & 170.6\footnotesize{$\pm$0.3} & 269.8\footnotesize{$\pm$0.1} & 375.6\footnotesize{$\pm$0.1} \\
        & RS & 53.8\footnotesize{$\pm$0.0} & 87.5\footnotesize{$\pm$0.1} & 167.3\footnotesize{$\pm$0.3} & 262.5\footnotesize{$\pm$0.5} & 359.3\footnotesize{$\pm$0.9} \\
        & GUARD & 32.0\footnotesize{$\pm$0.1} & 42.7\footnotesize{$\pm$0.2} & 53.8\footnotesize{$\pm$0.2} & 59.7\footnotesize{$\pm$0.4} & 63.0\footnotesize{$\pm$0.2} \\
        & \CCG\ours{} & \CCG53.9\footnotesize{$\pm$0.0} & \CCG88.1\footnotesize{$\pm$0.1} & \CCG169.2\footnotesize{$\pm$0.6} & \CCG268.0\footnotesize{$\pm$0.7} & \CCG373.4\footnotesize{$\pm$0.7} \\
        \midrule
        \multirow{4}{*}{\makecell[c]{GPT-Neo~2.7B\\(81.8\%)}}
        & Base & 106.2\footnotesize{$\pm$3.8} & 178.1\footnotesize{$\pm$0.2} & 312.0\footnotesize{$\pm$0.1} & 486.7\footnotesize{$\pm$0.2} & 677.6\footnotesize{$\pm$0.1} \\
        & RS & 106.8\footnotesize{$\pm$3.7} & 172.6\footnotesize{$\pm$0.5} & 300.5\footnotesize{$\pm$0.6} & 462.2\footnotesize{$\pm$0.5} & 630.0\footnotesize{$\pm$1.7} \\
        & GUARD & 43.4\footnotesize{$\pm$0.6} & 52.8\footnotesize{$\pm$0.3} & 60.6\footnotesize{$\pm$0.0} & 65.0\footnotesize{$\pm$0.4} & 67.6\footnotesize{$\pm$0.1} \\
        & \CCG\ours{} & \CCG106.7\footnotesize{$\pm$0.1} & \CCG175.4\footnotesize{$\pm$0.4} & \CCG308.9\footnotesize{$\pm$0.7} & \CCG481.4\footnotesize{$\pm$0.9} & \CCG670.6\footnotesize{$\pm$1.7}\\
        \midrule
        \multirow{4}{*}{\makecell[c]{Pythia~6.9B\\(82.2\%)}}
        & Base & 47.1\footnotesize{$\pm$0.5} & 91.9\footnotesize{$\pm$1.2} & 167.6\footnotesize{$\pm$0.9} & 223.8\footnotesize{$\pm$3.0} & 301.5\footnotesize{$\pm$9.3} \\
        & RS & 47.1\footnotesize{$\pm$0.1} & 89.6\footnotesize{$\pm$1.0} & 161.8\footnotesize{$\pm$4.5} & 220.8\footnotesize{$\pm$1.9} & 284.8\footnotesize{$\pm$13.5} \\
        & GUARD & 29.9\footnotesize{$\pm$0.2} & 42.1\footnotesize{$\pm$0.3} & 52.1\footnotesize{$\pm$0.1} & 57.0\footnotesize{$\pm$0.3} & 61.9\footnotesize{$\pm$0.2} \\
        & \CCG\ours{} & \CCG47.1\footnotesize{$\pm$0.0} & \CCG90.0\footnotesize{$\pm$0.5} & \CCG166.4\footnotesize{$\pm$1.0} & \CCG223.4\footnotesize{$\pm$6.6} & \CCG308.2\footnotesize{$\pm$4.1} \\
        \bottomrule[1.2pt]
    \end{tabular}
\end{table*}

\subsection{Perplexity and Violation Rate on Main Tasks}
\label{NCO:app:ssec:quality_and_violation_analysis}
Tables~\ref{NCO:app:tab:profanity_suppression_various_metric}
and~\ref{NCO:app:tab:pii_suppression_various_metric}
report perplexity and violation rates 
for the two main tasks.
These results complement the main relative throughput analysis 
and the absolute throughput results.
They show that the throughput trends observed in the main paper 
do not come from a degradation 
in constraint satisfaction or generation quality.

For RTP profanity suppression, 
the unconstrained models produce forbidden strings 
with violation rates between $11.6\%$ and $17.0\%$.
All constrained methods reduce the violation rate to $0.0\%$.
For Enron PII suppression, 
the unconstrained models show high violation rates 
across all three Pile-trained models.
Rejection sampling and \ours{} enforce the original regex constraints 
and reduce the violation rate to $0.0\%$.
GUARD cannot directly enforce the original regex constraints, 
so it is evaluated using the finite-string approximation 
described in Appendix~\ref{NCO:app:ssec:baseline_implementation_details}.
Perplexity remains at a similar level across constrained methods, 
which indicates that the constraints do not introduce 
substantial quality degradation under the evaluated settings.

\begin{table*}[th]
\centering
\caption{Quality and violation evaluation results for profanity suppression 
with finite hard constraints at batch size $1$.
All constrained methods reduce the violation rate to 0.0\%, while \ours{}
preserves throughput close to the base model across model families.}
\label{NCO:app:tab:profanity_suppression_various_metric}
\begin{tabular}{lllccc}
\toprule
Model & Size & Method & Perplexity~($\downarrow$) & Violation Rate~($\downarrow$) & Throughput~($\uparrow$) \\
\midrule
\multirow{4}{*}{Llama 2}
& \multirow{4}{*}{7B}
& Base & 1.75 & 14.80 & 50.31\small{$\pm$0.02} \\
& & Rejection Sampling & 1.75 & 0.00 & 50.00\small{$\pm$0.01} \\
& & GUARD & 1.76 & 0.00 & 35.29\small{$\pm$0.14} \\
& & \ours{} & 1.77 & 0.00 & 50.29\small{$\pm$0.04} \\
\midrule
\multirow{4}{*}{Llama 3.1}
& \multirow{4}{*}{8B}
& Base & 1.57 & 15.00 & 44.98\small{$\pm$0.02} \\
& & Rejection Sampling & 1.57 & 0.00 & 44.73\small{$\pm$0.02} \\
& & GUARD & 1.57 & 0.00 & 17.34\small{$\pm$0.15} \\
& & \ours{} & 1.57 & 0.00 & 44.94\small{$\pm$0.03} \\
\midrule
\multirow{4}{*}{Qwen2.5}
& \multirow{4}{*}{7B}
& Base & 2.18 & 17.00 & 46.06\small{$\pm$0.56} \\
& & Rejection Sampling & 2.21 & 0.00 & 46.29\small{$\pm$0.04} \\
& & GUARD & 2.19 & 0.00 & 16.37\small{$\pm$0.05} \\
& & \ours{} & 2.18 & 0.00 & 46.38\small{$\pm$0.08} \\
\midrule
\multirow{4}{*}{Falcon}
& \multirow{4}{*}{7B}
& Base & 1.65 & 11.60 & 47.72\small{$\pm$0.70} \\
& & Rejection Sampling & 1.64 & 0.00 & 48.13\small{$\pm$0.06} \\
& & GUARD & 1.64 & 0.00 & 25.52\small{$\pm$0.32} \\
& & \ours{} & 1.65 & 0.00 & 48.19\small{$\pm$0.06} \\
\midrule
\multirow{4}{*}{Falcon3}
& \multirow{4}{*}{7B}
& Base & 2.98 & 14.60 & 45.77\small{$\pm$0.63} \\
& & Rejection Sampling & 2.96 & 0.00 & 46.08\small{$\pm$0.09} \\
& & GUARD & 2.97 & 0.00 & 16.72\small{$\pm$0.13} \\
& & \ours{} & 3.12 & 0.00 & 46.15\small{$\pm$0.03} \\
\midrule
\multirow{4}{*}{Phi4}
& \multirow{4}{*}{14B}
& Base & 2.15 & 16.60 & 23.75\small{$\pm$0.19} \\
& & Rejection Sampling & 2.13 & 0.00 & 23.86\small{$\pm$0.07} \\
& & GUARD & 2.13 & 0.00 & 14.54\small{$\pm$0.19} \\
& & \ours{} & 2.17 & 0.00 & 23.81\small{$\pm$0.03} \\
\bottomrule
\end{tabular}
\end{table*}

For Enron PII suppression, the unconstrained models show high violation rates 
across all three Pile-trained models. 
GPT-J~6B, GPT-Neo~2.7B, and Pythia~6.9B all produce PII-like substrings 
in more than 80\% of responses. 
Both rejection sampling and \ours{} reduce the violation rate to 0.0\%. 
\ours{} also preserves throughput comparable to rejection sampling 
across all three models, indicating that the regex masking overhead remains small 
even when the base model frequently assigns high probability 
to invalid continuations.

\begin{table}[ht]
\centering
\caption{
Quality and violation evaluation results for PII suppression 
with forbidden regex constraints at batch size $1$.
The base models frequently generate PII patterns, while rejection sampling and
\ours{} reduce the violation rate to 0.0\% with similar throughput.}
\label{NCO:app:tab:pii_suppression_various_metric}
\begin{tabular}{lllccc}
\toprule
Model & Size & Method & Perplexity~($\downarrow$) & Violation Rate~($\downarrow$) & Throughput~($\uparrow$) \\
\midrule
\multirow{4}{*}{GPT-J}
& \multirow{4}{*}{6B}
& Base & 2.43 & 87.00 & 54.24\small{$\pm$0.03} \\
& & Rejection Sampling & 2.48 & 0.00 & 53.78\small{$\pm$0.03} \\
& & GUARD & 2.40 & 80.60 & 32.04\small{$\pm$0.09} \\
& & \ours{} & 2.52 & 0.00 & 53.87\small{$\pm$0.04} \\
\midrule
\multirow{4}{*}{GPT-Neo}
& \multirow{4}{*}{2.7B}
& Base & 2.18 & 81.80 & 106.18\small{$\pm$3.82} \\
& & Rejection Sampling & 2.22 & 0.00 & 106.78\small{$\pm$0.03} \\
& & GUARD & 2.14 & 79.40 & 43.41\small{$\pm$0.59} \\
& & \ours{} & 2.25 & 0.00 & 106.66\small{$\pm$0.08} \\
\midrule
\multirow{4}{*}{Pythia}
& \multirow{4}{*}{6.9B}
& Base & 2.09 & 82.20 & 47.09\small{$\pm$0.50} \\
& & Rejection Sampling & 2.13 & 0.00 & 47.09\small{$\pm$0.06} \\
& & GUARD & 2.11 & 81.80 & 29.93\small{$\pm$0.21} \\
& & \ours{} & 2.13 & 0.00 & 47.09\small{$\pm$0.02} \\
\bottomrule
\end{tabular}
\end{table}

\subsection{Sensitivity Analysis}
\label{NCO:app:ssec:sensitivity_analysis}

We conduct sensitivity analyses to evaluate 
whether \ours{} remains scalable 
as the constraint specification becomes larger or more complex. 
For finite hard constraints, we run the analysis using Llama~2~7B on RTP. 
We vary the number of forbidden strings from 1 to 32 
and the length of forbidden strings from 1 to 8. 
For regex constraints, we run the analysis using GPT-Neo~2.7B on Enron. 
We vary the number of DFA constraints, 
the total number of DFA states, 
and the total number of DFA edges from 1 to 32. 

The finite hard constraint experiments stress 
two different sources of constraint complexity. 
Increasing the number of forbidden strings enlarges the Aho-Corasick trie 
and increases the number of partial matches represented by the automaton. 
Increasing the forbidden string length mainly increases trie depth 
and the amount of token-level transition information needed during preprocessing. 
Once preprocessing is complete, 
\ours{} still maintains a single matching state per generated sequence. 
This keeps per-token decoding overhead small 
even when the forbidden lexicon is larger or contains longer strings.

The regex-constraint experiments focus on the structure of the compiled DFAs. 
The number of DFA constraints measures 
how many automata must be simulated in parallel. 
The total number of DFA states and edges measures the size of the automata 
after compilation. 
These quantities affect active-state updates, 
suffix-reachable state precomputation, and mask aggregation. 
The results show that the parallel DFA simulation remains practical 
as regex constraints become structurally larger.

\begin{figure}[ht]
    \centering
    \begin{subfigure}[t]{0.49\linewidth}
        \centering
        \includegraphics[width=\linewidth]{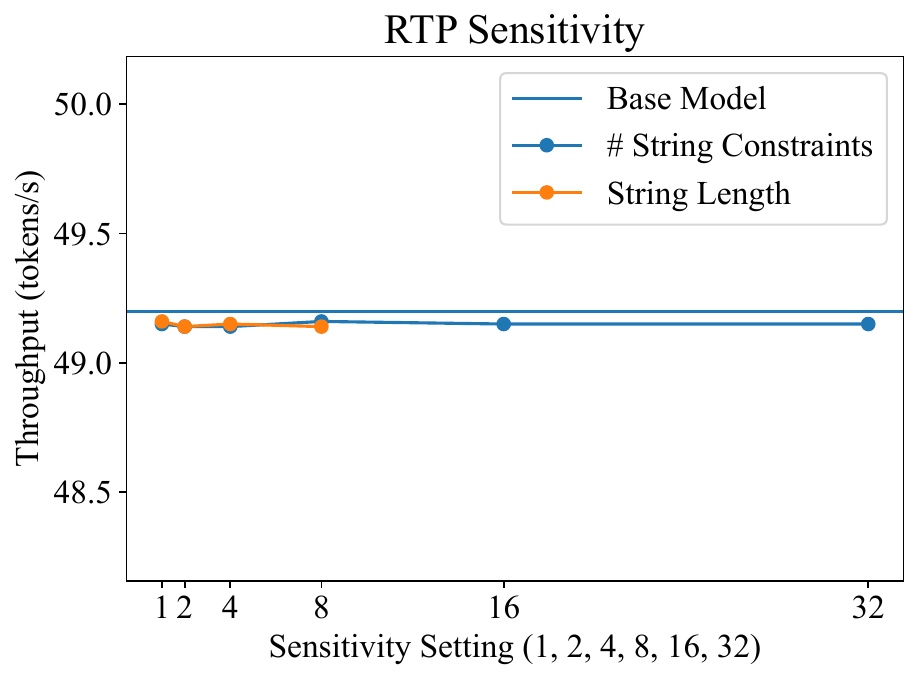}
    \end{subfigure}
    \hfill
    \begin{subfigure}[t]{0.49\linewidth}
        \centering
        \includegraphics[width=\linewidth]{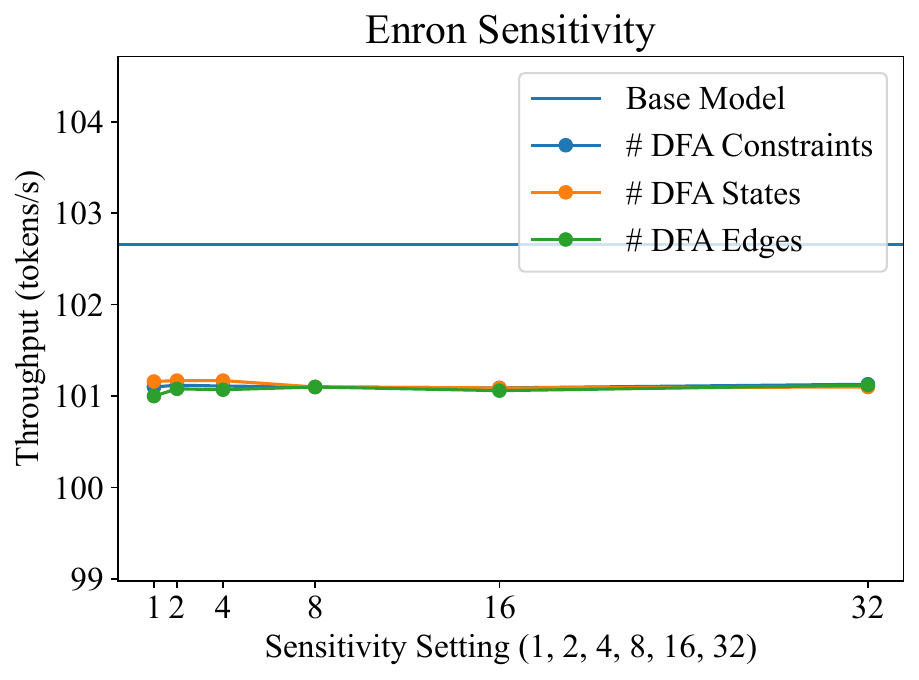}
    \end{subfigure}
    \caption{Sensitivity analysis of throughput under increasing constraint scale.}
    \label{NCO:app:fig:sensitivity_throughput}
\end{figure}

Figure~\ref{NCO:app:fig:sensitivity_throughput} reports 
the throughput of \ours{} under increasing constraint scale.
For finite hard constraints, throughput remains nearly unchanged 
as the forbidden lexicon grows in size or length. 
For regex constraints, throughput also remains stable 
as the number and size of DFAs increase. 
These results suggest that the runtime cost of \ours{} is not highly sensitive 
to the tested increases in constraint scale. 
This behavior is consistent with its design, 
where constraint-specific computation is largely moved to preprocessing 
and decoding uses compact online states and precomputed token-level masks. 
Overall, \ours{} maintains low decoding overhead across the tested settings 
while enforcing negative constraints throughout generation.

\subsection{Soft Negative Constraints}
\label{NCO:app:ssec:soft_negative_constraints}

Hard masking is appropriate when constraint satisfaction is mandatory, 
such as blocking explicit forbidden strings or structured PII patterns.
However, some applications such as style shifting may require softer control, 
where the goal is to discourage undesirable patterns 
rather than completely rule them out.
\ours{} supports this setting 
through the same online constraint states used for hard masking.
Given a penalty $\lambda > 0$, 
tokens that would complete a forbidden string or regex match 
receive an additive logit penalty of $-\lambda$ 
instead of being removed from the candidate set.
The case $\lambda=\infty$ corresponds to hard masking.

We evaluate this soft constraint mode 
as an additional analysis of the trade-off 
between constraint strength, generation quality, and throughput.
For RTP, we use Llama~3.1~8B with the same finite hard constraints 
as in the profanity suppression task.
For Enron, we use GPT-J~6B with the same regex constraints 
as in the PII suppression task.
We use sampling-based decoding with batch size $1$ 
because finite logit penalties are meaningful 
only when the decoder samples from the modified distribution.
We keep the sampling hyperparameters fixed 
across the Base model and all finite-penalty variants.
We vary $\lambda$ over 
$\{0.5,1,2,4,8,16,\infty\}$.
Finite values of $\lambda$ provide probabilistic suppression 
and do not guarantee zero violations.
The hard setting with $\lambda=\infty$ removes invalid continuations 
and recovers the hard-constraint behavior of \ours{}.
Rejection sampling is not included in this experiment.
It enforces constraints through a binary accept-or-reject decision 
and does not provide a finite-strength logit penalty.
GUARD is not included in this soft-penalty experiment
since this experiment is an ablation of \ours{} 
that replaces hard masking with a finite logit penalty.
\ours{} supports a continuous range of suppression strengths 
by reusing the same constraint states and token-level masks 
with different logit penalties.
Table~\ref{NCO:app:tab:soft_logit_constraints} 
reports the resulting trade-off among perplexity, violation rate, and throughput.

\begin{table}[ht]
\centering
\caption{Soft logit constraint results. We compare the base model with \ours{}
under different logit penalties. The hard constraint setting corresponds to
$\lambda=\infty$.}
\label{NCO:app:tab:soft_logit_constraints}
\begin{tabular}{lcccc}
\toprule[1.2pt]
Task & Penalty & Perplexity~($\downarrow$) & Violation Rate~($\downarrow$) & Throughput~($\uparrow$) \\
\midrule
\multirow{8}{*}{RTP} & Base & 2.26\small{$\pm$0.02} & 18.00\small{$\pm$0.53} & 44.64 \\
\cmidrule(lr){2-5}
 & $0.5$ & 2.29\small{$\pm$0.01} & 15.67\small{$\pm$0.70} & 44.60\small{$\pm$0.01} \\
 & $1$ & 2.26\small{$\pm$0.02} & 12.40\small{$\pm$0.20} & 44.59\small{$\pm$0.01} \\
 & $2$ & 2.29\small{$\pm$0.03} & 9.07\small{$\pm$0.31} & 44.58\small{$\pm$0.01} \\
 & $4$ & 2.30\small{$\pm$0.02} & 4.47\small{$\pm$0.90} & 44.58\small{$\pm$0.01} \\
 & $8$ & 2.32\small{$\pm$0.04} & 0.60\small{$\pm$0.20} & 44.58\small{$\pm$0.01} \\
 & $16$ & 2.29\small{$\pm$0.04} & 0.00\small{$\pm$0.00} & 44.61\small{$\pm$0.01} \\
 & $\infty$ & 2.29\small{$\pm$0.04} & 0.00\small{$\pm$0.00} & 44.60\small{$\pm$0.02} \\
\midrule
\multirow{8}{*}{Enron} & Base & 4.61\small{$\pm$0.06} & 84.07\small{$\pm$0.99} & 54.22\small{$\pm$0.03} \\
\cmidrule(lr){2-5}
 & $0.5$ & 4.57\small{$\pm$0.04} & 84.53\small{$\pm$0.64} & 53.89\small{$\pm$0.01} \\
 & $1$ & 4.51\small{$\pm$0.03} & 82.93\small{$\pm$1.03} & 53.92\small{$\pm$0.03} \\
 & $2$ & 4.50\small{$\pm$0.05} & 83.13\small{$\pm$1.40} & 53.94\small{$\pm$0.03} \\
 & $4$ & 4.46\small{$\pm$0.01} & 83.00\small{$\pm$1.06} & 53.93\small{$\pm$0.02} \\
 & $8$ & 4.55\small{$\pm$0.03} & 66.20\small{$\pm$1.04} & 53.94\small{$\pm$0.05} \\
 & $16$ & 4.80\small{$\pm$0.05} & 2.93\small{$\pm$0.23} & 53.92\small{$\pm$0.01} \\
 & $\infty$ & 4.83\small{$\pm$0.10} & 0.00\small{$\pm$0.00} & 53.92\small{$\pm$0.02} \\
\bottomrule[1.2pt]
\end{tabular}
\end{table}

\subsection{Scaling Regex Constraints under High Rejection Rates}

We further evaluate scalability under a more challenging regex constraint setting.
The goal of this experiment is to test how decoding methods behave 
when many regex constraints are enforced simultaneously 
and invalid continuations become frequent.
Unlike RTP and Enron, which are tied to specific suppression tasks,
this experiment uses prompts from Alpaca~\citep{TaoriGZDLGLH2023}  
to reflect a more general open ended generation scenario.
We sample 500 prompts from Alpaca 
and evaluate Llama~3.1~8B under regex constraints 
derived from RegexPSPACE~\citep{JinHH2025}.

RegexPSPACE contains 1685 regex constraints in total.
We filter out regexes that accept the empty string $\epsilon$,
since such constraints would reject every possible continuation 
from the beginning of decoding.
The original regexes are defined over the alphabet $\{a,b,c,d\}$.
We replace this alphabet with $\{w,x,y,z\}$ before evaluation.
This substitution prevents the constraints 
from suppressing an excessively large portion of ordinary text 
while preserving the structure of the regexes.
We vary the number of regex constraints 
over $\{10,25,50,100,250,500,668\}$.
The base model is evaluated once.
Rejection sampling and \ours{} are evaluated three times 
for each constraint scale.

\begin{table}[ht]
    \centering
    \caption{Results on the larger regex experiment with Llama 3.1 8B. The regex constraints are drawn from RegexPSPACE, where 668 regexes are retained after filtering 1685 benchmark instances.}
    \label{NCO:app:tab:scaling_regex_constraints}
    \begin{tabular}{lcc}
        \toprule[1.2pt]
        Method & \# Constraints & Throughput~($\uparrow$)\\
        \midrule
        base & - & 44.98\small{$\pm$0.02} \\ 
        \midrule
        \multirow{7}{*}{Rejection Sampling} & 10 & 44.65\small{$\pm$0.03}\\
         & 25 & 44.57\small{$\pm$0.01}\\
         & 50 & 44.48\small{$\pm$0.03}\\
         & 100 & 44.28\small{$\pm$0.04}\\
         & 250 & 43.67\small{$\pm$0.05}\\
         & 500 & 42.56\small{$\pm$0.05}\\
         & 668 & 41.76\small{$\pm$0.05}\\
        \midrule
        \multirow{7}{*}{\ours{}} & 10 & 44.68\small{$\pm$0.05}\\
         & 25 & 44.67\small{$\pm$0.04}\\
         & 50 & 44.65\small{$\pm$0.01}\\
         & 100 & 44.56\small{$\pm$0.03}\\
         & 250 & 44.22\small{$\pm$0.00}\\
         & 500 & 43.72\small{$\pm$0.05}\\
         & 668 & 43.37\small{$\pm$0.03}\\
        \bottomrule[1.2pt]
    \end{tabular}
\end{table}

Table~\ref{NCO:app:tab:scaling_regex_constraints} reports the results.
As the number of regex constraints increases,
both constrained methods face a harder decoding problem 
because more candidate tokens can complete a forbidden match.
The throughput of \ours{} also decreases under the largest constraint sets,
showing that the cost of aggregating masks 
across many regex constraints is not negligible in this regime.
The decrease is nevertheless much smaller than that of rejection sampling.
Rejection sampling suffers a sharp throughput drop 
as invalid candidates become more likely,
since it may need to repeatedly sample and reject tokens 
at the same decoding step.
\ours{} avoids this repeated trial process 
by computing invalid token masks before token selection.
These results show that \ours{} remains substantially more scalable 
than rejection sampling under high rejection rates,
even when enforcing hundreds of regex constraints during generation.

\subsection{Comparison with Modern Structured Decoding Frameworks}
\label{NCO:app:ssec:structured_decoding_comparison}

We compare \ours{} with XGrammar~\citep{DongRCLXZC2025} 
and Outlines~\citep{WillardL2023}.
Unlike \ours{}, which directly operates on factored negative constraints,
these frameworks require an equivalent positive representation 
of the global admissible language.
Deriving such a representation from multiple forbidden regexes 
can itself be nontrivial, 
since every substring occurrence must be excluded 
and the resulting automaton may grow substantially.

On Enron, we evaluate the four PII regex constraints using GPT-J~6B
on 500 prompts and generate up to 128 tokens per prompt.
We construct the corresponding global admissible grammar for XGrammar and Outlines, 
and measure active-token throughput at batch sizes 1, 8, and 32
together with the constraint-dependent setup time.
Table~\ref{NCO:app:tab:structured_decoding_enron} shows that
XGrammar achieves throughput comparable to \ours{} 
when the resulting global representation remains compact, 
while Outlines incurs larger overhead under batched decoding.

We additionally evaluate 403 forbidden strings on RTP with GPT-J~6B
over 500 prompts.
XGrammar is tested using both a constructed global grammar 
and its native finite-string exclusion interface.
As shown in Table~\ref{NCO:app:tab:structured_decoding_rtp},
\ours{} achieves higher steady-state throughput, 
whereas XGrammar requires less setup time.
The Outlines run produced predominantly EOS-only outputs 
and is therefore omitted from the comparison.

\begin{table}[ht]
\centering
\caption{Structured-decoding comparison on the Enron PII suppression task with
four forbidden regex constraints, GPT-J~6B, and 500 prompts.
Throughput is measured in active tokens per second.
Setup time includes constraint-dependent preprocessing and global
construction when applicable.}
\label{NCO:app:tab:structured_decoding_enron}
\small
\begin{tabular}{llrrrrr}
\toprule[1.2pt]
\multirow{2}{*}[-0.5ex]{Method} & \multirow{2}{*}[-0.5ex]{Representation} & \multicolumn{3}{c}{Throughput} & \multirow{2}{*}[-0.5ex]{Setup Time~(s)} & \multirow{2}{*}[-0.5ex]{Violation} \\
\cmidrule{3-5}
 &  & B=1 & B=8 & B=32 &  \\
\midrule
Base
    & Unconstrained
    & 81.4 & 365 & 467 & 0 & 80.2--81.2\% \\
\ours{}
    & Factored regexes
    & 80.4 & 378 & 487 & 3.20--3.25 & 0\% \\
XGrammar
    & Global grammar
    & 80.4 & 378 & 485 & 0.395--0.398 & 0\% \\
Outlines
    & Global grammar
    & 67.6 & 126 & 106 & 0.923--1.10 & 0\% \\
\bottomrule[1.2pt]
\end{tabular}
\end{table}

\begin{table}[ht]
\centering
\caption{Structured-decoding comparison on the RTP profanity suppression task 
with 403 forbidden strings, GPT-J~6B, and 500 prompts.
Throughput is measured in active tokens per second.
Setup time includes constraint-dependent preprocessing and global construction
when applicable.}
\label{NCO:app:tab:structured_decoding_rtp}
\small
\begin{tabular}{llrrrrr}
\toprule[1.2pt]
\multirow{2}{*}[-0.5ex]{Method} & \multirow{2}{*}[-0.5ex]{Representation} & \multicolumn{3}{c}{Throughput} & \multirow{2}{*}[-0.5ex]{Setup Time~(s)} & \multirow{2}{*}[-0.5ex]{Violation} \\
\cmidrule{3-5}
 &  & B=1 & B=8 & B=32 &  \\
\midrule
Base
    & Unconstrained
    & 82.3 & 580 & 1,810 & 0 & 19.0--20.0\% \\
\ours{}
    & Factored strings
    & 82.2 & 581 & 1,820 & 23.2--23.4 & 0\% \\
XGrammar
    & Global grammar
    & 80.6 & 545 & 1,570 & 3.73--3.75 & 0\% \\
XGrammar
    & Native string exclusion
    & 81.0 & 555 & 1,640 & 0.775--0.789 & 0\% \\
\bottomrule[1.2pt]
\end{tabular}
\end{table}

\subsection{State Growth in Global Automaton Construction}
\label{NCO:app:ssec:state_growth}

We measure the state growth incurred when a factored set of negative regex
constraints is converted into a single global avoidance automaton.
We use the controlled constraint family
\[
R = \bigcup_{i=1}^{K} R_i,
\qquad
R_i = s_i(\Sigma \setminus \{e_i\})^{*}e_i.
\]
For each $K$, we construct the DFA for the combined forbidden language $R$
and determinize the corresponding substring-occurrence language
$\Sigma^{*}R\Sigma^{*}$.

Table~\ref{NCO:app:tab:global_state_growth} reports the resulting state
counts, construction time, storage, and peak memory.
We refer to states that can be visited 
without violating the constraint as safe states.
At $K=20$, the forbidden language has only 62 DFA states, whereas
determinizing $\Sigma^{*}R\Sigma^{*}$ produces 54,525,952 states and
requires 43.0 GB of peak memory.
This result illustrates the state-growth issue avoided by the factored
representation used in \ours{}.

\begin{table}[ht]
\centering
\caption{Cost of constructing a global violation-free representation for the
controlled regex family
$R=\bigcup_{i=1}^{K}s_i(\Sigma\setminus\{e_i\})^{*}e_i$.}
\label{NCO:app:tab:global_state_growth}
\small
\begin{tabular}{rrrrrrr}
\toprule[1.2pt]
$K$
& $|\text{DFA}(R)|$
& $|\text{DFA}(\Sigma^{*}R\Sigma^{*})|$
& \# Safe States
& Construction (s)
& Dense Storage
& Peak Memory \\
\midrule
2  & 8  & 28         & 4         & 0.000643 & 560 B   & 24.4 MB \\
6  & 20 & 1,088      & 64        & 0.0165   & 56.6 KB & 24.9 MB \\
10 & 32 & 27,648     & 1,024     & 0.503    & 2.32 MB & 39.4 MB \\
14 & 44 & 606,208    & 16,384    & 13.9     & 70.3 MB & 408 MB \\
16 & 50 & 2,752,512  & 65,536    & 73.2     & 363 MB  & 1.98 GB \\
18 & 56 & 12,320,768 & 262,144   & 373      & 1.82 GB & 9.25 GB \\
20 & 62 & 54,525,952 & 1,048,576 & 1,870    & 8.94 GB & 43.0 GB \\
\bottomrule[1.2pt]
\end{tabular}
\end{table}

\subsection{Precomputation and Memory Overhead at Scale}
\label{NCO:app:ssec:precomputation_scale}

We measure the precomputation time, table storage, and masking latency of
\ours{} as the number of regex constraints increases.
The vocabulary size is fixed at 50,400, and masking latency is measured at
batch size 8.
The 1,000-constraint setting synthetically extends the constraint set.

Table~\ref{NCO:app:tab:precomputation_scale} reports the results.
As the number of constraints increases from 10 to 1,000, 
the total number of component DFA states grows from 32 to 3,239, 
with corresponding increases in precomputation time and storage.
At 668 constraints, masking requires 0.504 ms per decoding step, 
compared with 13.7 ms for model inference.
At 1,000 constraints, the dense tables require 1.469 GB of storage and
470 seconds of precomputation.

\begin{table}[ht]
\centering
\caption{Precomputation and runtime scaling of \ours{} under increasing numbers of
regex constraints.
Masking latency is measured at batch size 8 with a vocabulary size of 50,400.}
\label{NCO:app:tab:precomputation_scale}
\small
\begin{tabular}{rrrrr}
\toprule[1.2pt]
\# Constraints
& \# DFA States
& Precomputation~(s)
& Memory~(GB)
& Masking Latency~(ms/step) \\
\midrule
10    & 32    & 4.59 & 0.015 & 0.0227 \\
100   & 314   & 46.2 & 0.142 & 0.0503 \\
668   & 2,162 & 315  & 0.981 & 0.504 \\
1,000 & 3,239 & 470  & 1.469 & 0.859 \\
\bottomrule[1.2pt]
\end{tabular}
\end{table}

\subsection{Generation Behavior under Hard Masking and Soft Logit Penalties}
\label{NCO:app:ssec:generation_behavior}

We analyze the effect of hard masking and finite logit penalties using
GPT-J~6B on Enron with four PII regex constraints at batch size 8.
For greedy decoding, we compare the Base and hard-masking settings.
For sampling, we compare Base, logit shifts of $-5$ and $-10$, and hard
masking over seeds 42--44.
Sampling uses temperature 1.0, top-$p$ 1.0, and top-$k$ 0.
We report the probability mass affected by masking (\textit{Mass}),
the entropy change $\Delta H=H(q)-H(p)$, and
$\mathrm{KL}(q\|p)$.
We also report the original-model NLL over the next 1, 4, and 8 tokens
after a masking event and the repeated trigram rate within the following
16 tokens.

Table~\ref{NCO:app:tab:generation_behavior} reports the results.
Under sampling, hard masking reduces the violation rate from 81.8\% to
0\% while affecting 2.61\% of the probability mass.
The original-model NLL over the next 1, 4, and 8 tokens is
2.70, 2.06, and 1.84, respectively, and the repeated trigram rate remains
close to the Base setting (10.3\% versus 9.86\%).
Finite penalties reduce the distributional change but retain nonzero
violation rates.

\begin{table}[ht]
\centering
\caption{Generation behavior under hard and soft masking on Enron with GPT-J~6B.
Sampling results are averaged over seeds 42--44.
NLL and PPL are computed under the original GPT-J~6B distribution.
\textit{Mass} denotes the probability mass affected by the intervention,
and 3-gram repetition is measured over the following 16 tokens.}
\label{NCO:app:tab:generation_behavior}
\small
\begin{tabular}{llrrrrrrlr}
\toprule[1.2pt]
Decoding & Setting & Viol. & NLL & PPL & Mass & $\Delta H$ & KL & NLL (1/4/8) & 3-gram Rep. \\
\midrule
Greedy & Base & 80.8\% & 0.41 & 1.50 & 0\% & 0 & 0 & -- & 24.3\% \\
Greedy & Hard & 0\% & 0.56 & 1.75 & 2.24\% & 0.011 & 0.139 & 1.36/0.95/0.78 & 26.5\% \\
\midrule
Sampling & Base & 81.8\% & 1.48 & 4.38 & 0\% & 0 & 0 & -- & 9.86\% \\
Sampling & $-5$ & 77.1\% & 1.51 & 4.53 & 4.89\% & 0.003 & 0.020 & 0.87/1.08/1.13 & 9.09\% \\
Sampling & $-10$ & 33.8\% & 1.70 & 5.48 & 3.22\% & 0.011 & 0.106 & 2.27/1.90/1.75 & 9.29\% \\
Sampling & Hard & 0\% & 1.75 & 5.74 & 2.61\% & 0.012 & 0.146 & 2.70/2.06/1.84 & 10.3\% \\
\bottomrule[1.2pt]
\end{tabular}
\end{table}

% \section{Limitations and Broader Impacts}
\section{Limitations}
\label{NCO:app:ssec:limitations}
\ours{} is designed for explicit negative substring constraints.
It prevents user-specified forbidden hard constraints and regex patterns
from appearing in generated outputs,
but it does not address semantic variants 
that are not captured by the constraint specification.
For example, a regex for structured PII 
can block matching email addresses or phone numbers,
but it cannot prevent all semantically sensitive disclosures.
Similarly, a forbidden lexicon can block listed words,
but it may not cover paraphrases, misspellings, or context-dependent harmful content.
Thus, \ours{} should be viewed as a decoding-time mechanism 
for explicit pattern blocking,
rather than a standalone solution for broad semantic safety.

\ours{} intervenes by masking invalid next-token candidates during decoding.
This design ensures that forbidden patterns are not completed 
when the hard mask is used, but it only controls future token choices.
If a forbidden or undesirable span has already been generated 
before the constraint is applied,
\ours{} does not revise or remove that existing prefix.
In such cases, additional mechanisms such as input filtering, 
generation restart, or post-generation inspection may still be needed.

The regular-constraint component of \ours{} relies on constraints 
that can be compiled into finite automata.
This covers many practical regex patterns, including common structured patterns 
such as email addresses and phone numbers.
However, not every feature of modern regex engines corresponds directly 
to a regular language.
Features such as backreferences, lookaround assertions, 
or engine-specific matching semantics may require additional handling 
or may fall outside the finite automaton formulation used by \ours{}.
The practical applicability of regex constraints therefore depends on
the supported regex fragment and the compilation procedure.

Finally, \ours{} reduces decoding-time overhead 
by precomputing token-level transition and masking information 
for a given constraint set.
When the constraint specification is changed,
this precomputation must be updated accordingly.
This is suitable for deployment settings 
where constraint policies are reused across many generations,
but it can introduce extra cost 
when constraints are modified frequently or customized for each request.
Reducing the cost of dynamic constraint updates is an interesting direction 
for future work.

\end{document}